\documentclass{article}
\usepackage[letterpaper,bindingoffset=0.2in,%
left=1in,right=1in,top=1in,bottom=1in,%
footskip=.25in]{geometry}

\usepackage[utf8]{inputenc} 
\usepackage[T1]{fontenc}    
\usepackage{xcolor}         
\usepackage{thm-restate}
\usepackage{soul}
\usepackage{ifthen}
\usepackage{cite}
\usepackage{amsmath} 
\usepackage{url,color,cite}
\usepackage{caption}
\usepackage{amsmath,amsthm,amsfonts,amssymb}
\usepackage{cases}
\usepackage{mathtools}
\DeclarePairedDelimiter{\ceil}{\lceil}{\rceil}

\usepackage{float}
\usepackage{epstopdf}
\usepackage{nicefrac}
\usepackage{algorithm}
\usepackage[noend]{algpseudocode}
\usepackage{pifont}
\usepackage{subcaption}
\usepackage{sidecap}
\usepackage{multirow}
\usepackage{multicol}
\usepackage{bbm}

\newtheorem{assumption}{Assumption}
\newtheorem{lemma}{Lemma}
\newtheorem*{lemma*}{Lemma}
\newtheorem{theorem}{Theorem}
\newtheorem*{theorem*}{Theorem}
\theoremstyle{definition}
\newtheorem{definition}{Definition}
\theoremstyle{definition}

\newtheorem{remark}{Remark}
\theoremstyle{definition}

\theoremstyle{definition}

\newtheorem*{claim*}{Claim}

\usepackage{microtype}
\usepackage{graphicx}
\usepackage{booktabs} 
\usepackage{hyperref}
\usepackage{cleveref}
\usepackage{enumitem}

\usepackage{lipsum}

\newcommand{\added}[1]{\textcolor{purple}{#1}}

\begin{document}

\title{Differentially Private Stochastic Linear Bandits:\\ (Almost) for Free} 

\author{Osama A. Hanna\thanks{The first and second authors made equal contribution.}
\and Antonious M. Girgis\footnotemark[1]
\and Christina Fragouli
\and Suhas Diggavi
}
\date{University of California, Los Angeles, USA\\
	Email:\{ohanna, amgirgis, christina.fragouli, suhasdiggavi\}@ucla.edu
} 
\maketitle
\begin{abstract}
In this paper, we propose differentially private algorithms for the problem of stochastic linear bandits in the central, local and shuffled models. In the central model, we achieve almost the same regret as the optimal non-private algorithms, which means we get privacy for free. In particular, we achieve a regret of $\tilde{O}(\sqrt{T}+\frac{1}{\epsilon})$ matching the known lower bound for private  linear bandits, while the best previously known algorithm achieves  $\tilde{O}(\frac{1}{\epsilon}\sqrt{T})$. In the local case, we achieve a regret of $\tilde{O}(\frac{1}{\epsilon}{\sqrt{T}})$ which matches the non-private regret for constant $\epsilon$, but suffers a regret penalty when $\epsilon$ is small. In the shuffled model, we also achieve regret of $\tilde{O}(\sqrt{T}+\frac{1}{\epsilon})$ 
while the best previously known algorithm suffers a regret of $\tilde{O}(\frac{1}{\epsilon}{T^{3/5}})$. Our numerical evaluation validates our theoretical results. 
\end{abstract}

{\allowdisplaybreaks

\section{Introduction}\label{sec:intro}

Stochastic linear bandits offer a sequential decision framework where a learner interacts with an environment over rounds, and decides what is the optimal (from a potentially infinite set) action to play so as to achieve the best possible reward (minimize her regret). In particular, at each round, the learner may  take into account all past rewards and actions to decide the next action to play, and in return receive a new reward.   
This model has been widely adopted  both in theory but also in a number of applications, including recommendation systems, health, online education, and resource allocation~\cite{mary2015bandits, bouneffouf2017bandit, rafferty2018bandit,bouneffouf2019survey}.
Motivated by the fact that many of these applications are privacy-sensitive, in this paper we explore what is the performance in terms of regret we can achieve, if we are constrained to use a privacy-preserving stochastic linear bandit algorithm.  

In particular, in this paper we aim to design algorithms that preserve the privacy of the rewards, from an adversary that can observe all actions that the learner plays. We assume that the learner is connected through a secure communication channel with clients, who play the requested actions. For example, the central learner may make restaurant recommendations to mobile devices, may regulate the operation of on-body sensors in senior living communities, may decide what educational exercises to provide to students, or what jobs to allocate to workers. The actions the clients play - what restaurant is visited, which sensor is activated, what is the exercise solved, what is the job performed - may be naturally visible especially in public environments.  What we care to protect are  the rewards, that may capture private information, such as personal preferences in recommendation systems, health indices in online health, perfomance  in online education, and income gained in resource allocation. Our goal  is to design algorithms that preserve the privacy of the rewards, while still (almost) achieve the same regret as the traditional algorithms that do not take privacy into consideration.

We do so for three different setups,  depicted in Figure~\ref{fig:system_model},  in each case measuring the privacy using 
Differential Privacy (DP)  measures~\cite{Calibrating_DP06,dwork2014algorithmic}. In the {\bf central DP model}, the learner is a trusted server. The server employs a DP mechanism on aggregates of the reward realizations she collects, to ensure that the actions do not reveal information on the rewards. In the {\bf local DP model}, the learner is an untrusted server. The clients provide privatized rewards to the server, who then uses this noisy input to decide her next actions. In the {\bf shuffled model}, the learner is still an untrusted server, but now a trusted node, that can act as a relay in the communication between the clients and the server,  serves as a shuffler, and can randomly permute the privatized rewards before making them available to the server. A shuffler offers a privacy-amplification mechanism that has recently become popular in the literature, as it is easy to implement (simply takes a set of inputs and randomly permutes them), and may enable  better privacy-regret performance~\cite{cheu2019distributed,erlingsson2019amplification,balle2019privacy,feldman2022hiding,girgis2021renyi}.

\begin{figure*}[t!]
	\begin{subfigure}{0.3\linewidth}
		\centerline{\includegraphics[width=4.5cm, scale=0.5]{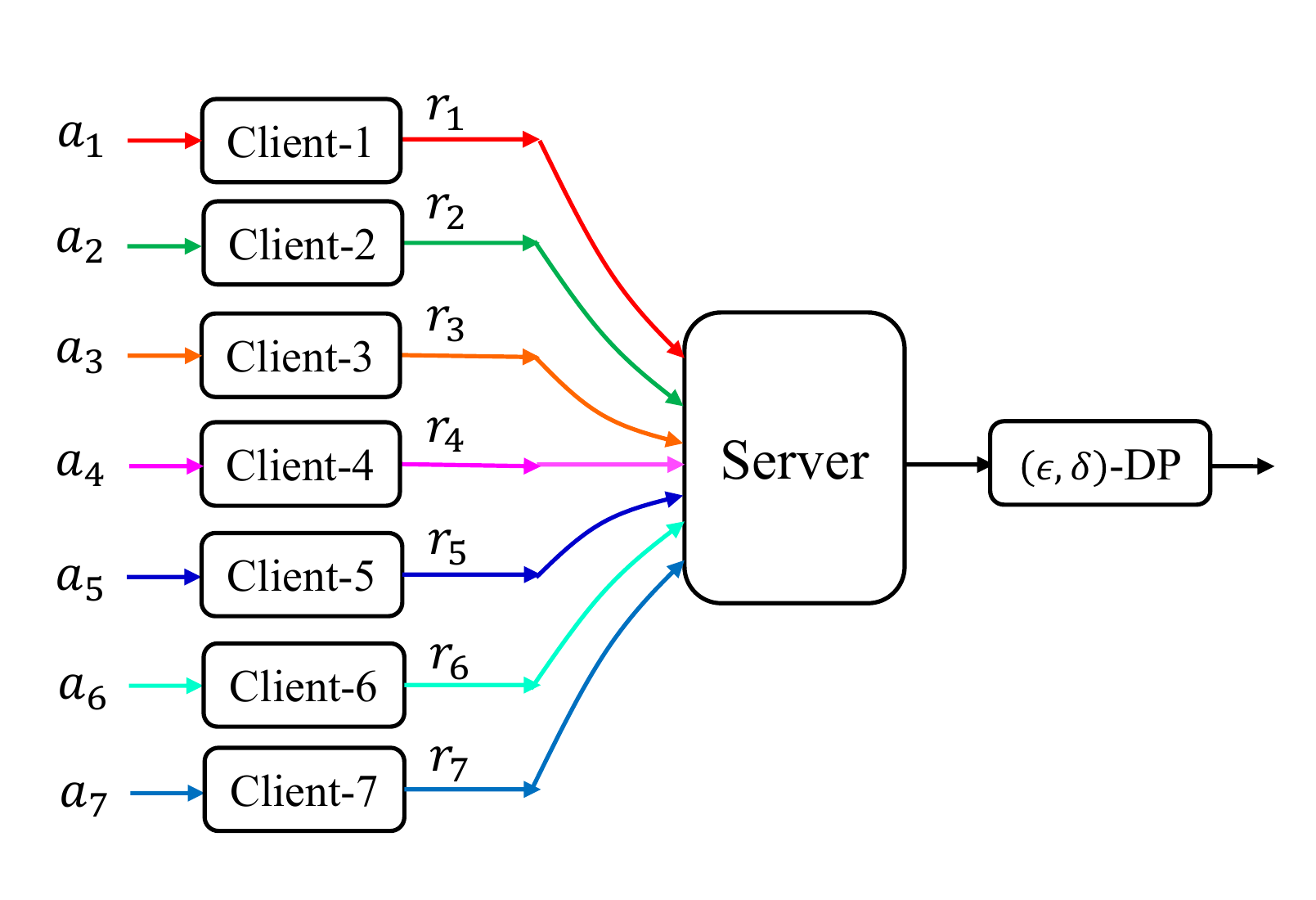}}
		\caption{Central model.}
		~\label{fig:central}
	\end{subfigure}
	\begin{subfigure}{0.3\linewidth}
		\centerline{\includegraphics[width=4.5cm, scale=0.5]{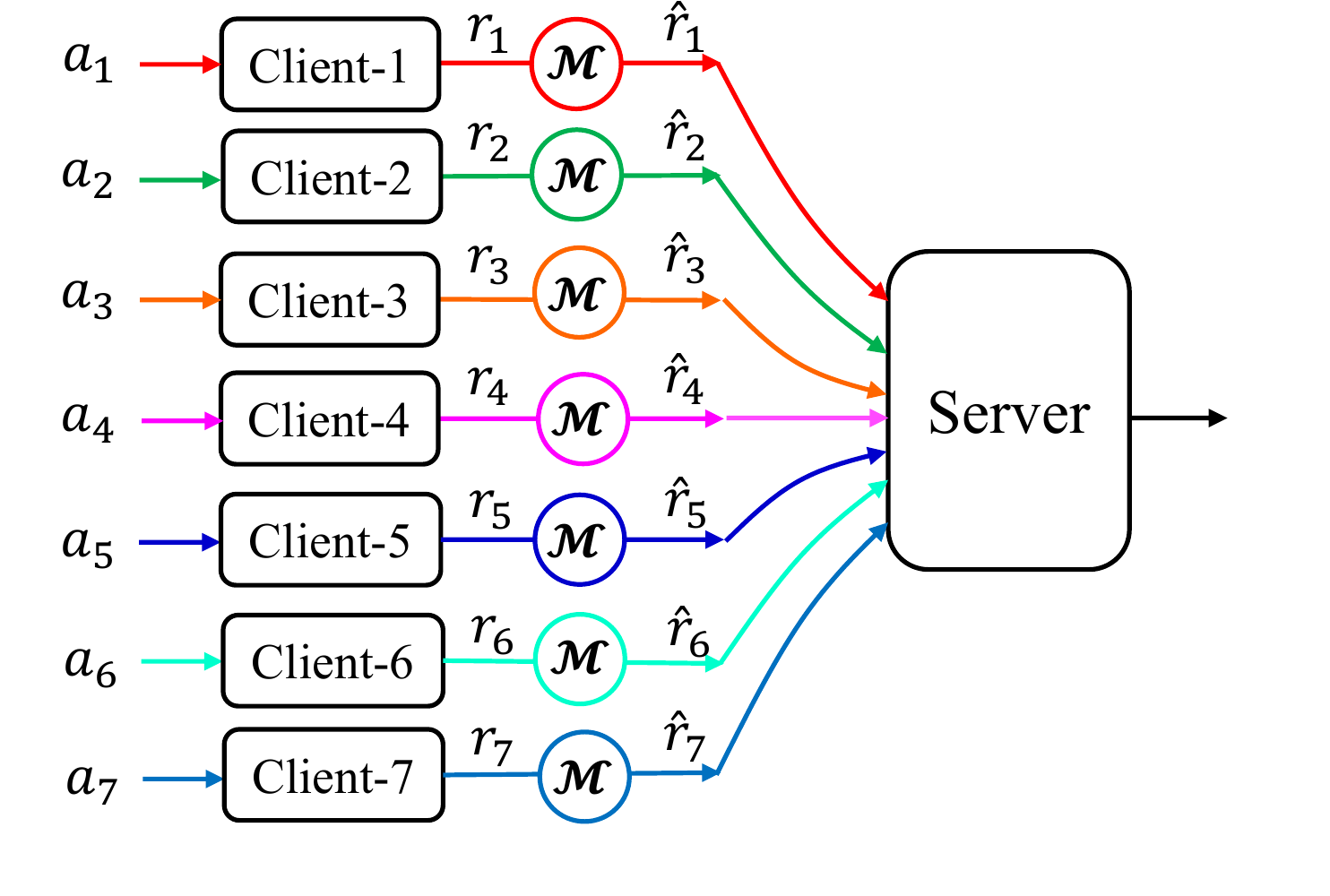}}
		\caption{Local model.}
		~\label{fig:local}
	\end{subfigure}
	\begin{subfigure}{0.3\linewidth}
		\centerline{\includegraphics[width=4.5cm, scale=0.5]{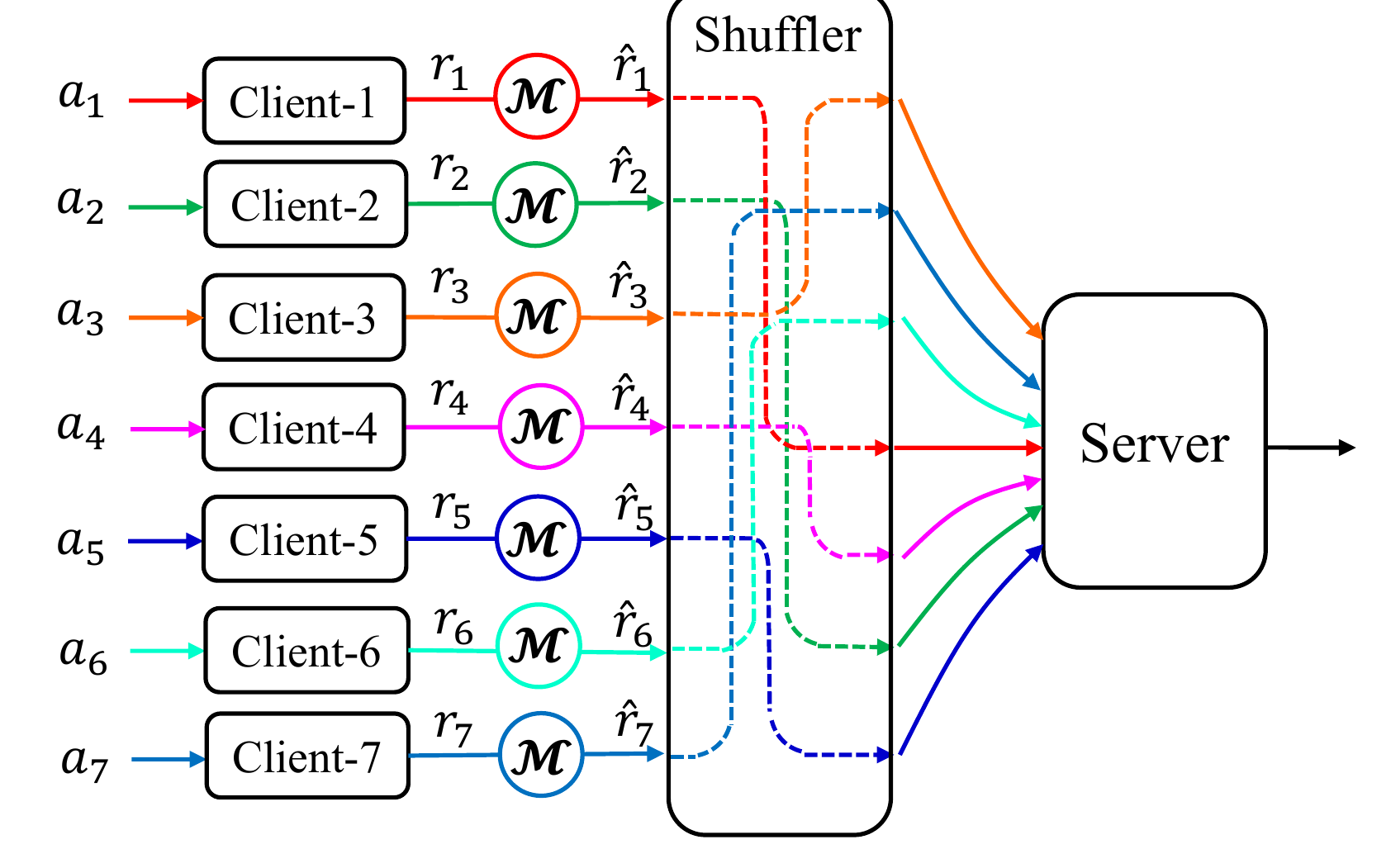}}
		\caption{Shuffled model.}
		~\label{fig:shuffle}
	\end{subfigure}
	\caption{\label{fig:system_model} In case (a) the server is trusted, and we ensure that the publicly observable actions maintain privacy of the rewards. In (b) and (c) we  maintain privacy from an untrusted server.}
	\vspace{-0.5cm}
\end{figure*}

Our main contributions are as follows.\\
$\bullet$ For the {\bf central DP model}, we design an algorithm that guarantees $\epsilon$-DP (see Definition~\ref{def:DP} in Section~\ref{sec:problem_formulation}) and achieves regret that matches existing lower bounds. In particular, over $T$ rounds, it achieves 	regret $R_T= O\left(\sqrt{T\log T}+\frac{\log^2 T}{\epsilon}\right)$ w.h.p., which is  optimal within a $\log T$ factor: a lower bound of $O(\sqrt{T})$ is proven in~\cite{rusmevichientong2010linearly} for non-private linear bandits, while a lower bound of $O(\frac{\log T}{\epsilon})$ is shown in~\cite{shariff2018differentially} for $\epsilon$-DP linear bandits. Note that for $\epsilon \approx 1$ (perhaps the most common case) the dominant term  $O(\sqrt{T\log T})$ matches the regret of the best known algorithms for the non-private case (eg., LinUCB~\cite{abbasi2011improved,rusmevichientong2010linearly}), and hence, we get privacy for free.\\
$\bullet$ For the {\bf local DP model}, we design an algorithm that guarantees $\epsilon_0$-LDP (see Definition~\ref{def:ldp} in Section~\ref{sec:problem_formulation}) and achieves regret $R_T=\mathcal{O}\left(\sqrt{T\log(T)}/\epsilon_0\right)$ w.h.p.; this regret matches the non-private regret for constant $\epsilon_0$, but suffers a regret  penalty when $\epsilon_0$ is small. Although our algorithm  does not improve the regret order as compared to the best-known
algorithm for private (contextual) linear bandits in~\cite{han2021generalized}, it offers an alternative approach that serves as a foundation for the shuffled case.  \\
$\bullet$ For the {\bf shuffled model}, we  leverage the help of a trusted shuffler to ensure both that the output of each client satisfies $\epsilon_0$-LDP and that the output of the secure shuffler satisfies $\epsilon$-DP requirements. Our algorithm achieves regret $R_T=\mathcal{O}\left(\sqrt{T\log(T)}+\frac{\log(T)}{\epsilon}\right)$ w.h.p. that matches the regret of the best non-private algorithms, same as the central model. Furthermore, our algorithm outperforms the best known algorithm for private (contextual) linear bandits in~\cite{garcelon2022privacy,chowdhury2022shuffle} that use shuffling.

Our results are summarized in Table~I, where we also provide known results in the literature (see also discussion next). To the best of our knowledge, in all three cases, our algorithms achieve the best currently known results for private linear bandits, significantly improving from the previously best known results in the case of the central and shuffled model, and closely matching in some cases existing lower bounds.

{\bf Our Work vs. Related Work.}
Differential Privacy (DP) algorithms have been proposed for the generic multi-armed bandits (MAB) problems  \cite{sajed2019optimal, ren2020multi, tenenbaum2021differentially}, yet these algorithms would not work well for linear bandits, as linear bandits allow for an infinite set of actions while generic MAB have a regret that increases with the number of actions.
Closer to ours is work on DP for contextual linear bandits \cite{shariff2018differentially, zheng2020locally, han2021generalized, garcelon2022privacy}; indeed, linear bandits can be viewed as  (a special case of) contextual linear bandit setup with a single context. 
The work in \cite{shariff2018differentially} considers contextual linear bandits with DP in a centralized setting and propose an algorithm that achieves a regret of $\tilde{O}(\sqrt{T}/\epsilon)$. This does not match the best known lower bound for the centralized setting of $\Omega(\sqrt{T}+\log (T)/\epsilon)$ \cite{shariff2018differentially}. Our work achieves the lower bound of $\Omega(\sqrt{T}+\log (T)/\epsilon)$ up to logarithmic factors for the special case of stochastic linear bandits. The work in \cite{zheng2020locally} considers contextual linear bandits with LDP, where the contexts can be adversarial. The work proposes an algorithm that achieves a regret of $\tilde{O}(T^{3/4}/\epsilon_0)$ and conjectures that the regret is optimal up to a logarithmic factor. The authors in \cite{han2021generalized} consider a special case, where the contexts are generated from a distribution, and propose a method that achieves a regret of $\tilde{O}(\sqrt{T}/\epsilon_0)$ under certain assumptions on the context distribution.
Our algorithm for the local model achieves the same regret order using an alternative method.
The works in \cite{garcelon2022privacy, chowdhury2022shuffle} consider contextual linear bandits in the shuffled model where the best known algorithm achieves a regret of $\tilde{O}(T^{3/5})$. Our proposed algorithms achieve a regret of $\tilde{O}(\sqrt{T}+1/\epsilon)$, matching the information theoretic lower bound in \cite{shariff2018differentially}, for stochastic linear bandits in the shuffled model. A summary of the best results for DP  contextual linear bandits and our results is presented in Table~\ref{T1}.

{\bf Paper organization.} We present the problem formulation  in Section~\ref{sec:problem_formulation}. We design and analyze privacy-preserving linear bandit algorithms
for the central model  in Section~\ref{sec:central}, for the local model  in Section~\ref{sec:LDP} and for the shuffled model in Section~\ref{sec:shuffled}. We provide numerical results in Section~\ref{sec:numerics}.

\begin{table}[t!]
	\centering
	\begin{tabular}{ |c|c|c|c|c| }
		\hline
		\multirow{2}{*}{Algorithm} & \multirow{2}{*}{Regret Bound} & \multirow{2}{*}{Context} & \multicolumn{2}{c|}{Privacy Model}   \\ 
		\cline{4-5} & & & Central DP & Local DP\\ 
		\hline\hline
		Central DP~\cite{shariff2018differentially}& $\tilde{\mathcal{O}}\left(\frac{\sqrt{T}}{\epsilon}\right)$ & Adversarial & $\left(\epsilon,\delta\right)$ & N/A \\
		LDP~\cite{zheng2020locally}& $\tilde{\mathcal{O}}\left(\frac{T^{3/4}}{\epsilon_0}\right)$ & Adversarial & $\left(\epsilon=\epsilon_0,\delta\right)$ & $\left(\epsilon_0,\delta\right)$ \\ 
		LDP+shuffling~\cite{garcelon2022privacy}& $\tilde{\mathcal{O}}\left(\frac{T^{2/3}}{\epsilon^{1/3}}\right)$ & Adversarial & $\left(\epsilon,\delta\right)$ & $\left(\epsilon_0=\epsilon^{2/3} T^{1/6},\delta\right)$\\ 
		LDP~\cite{han2021generalized}& $\tilde{\mathcal{O}}\left(\frac{\sqrt{T}}{\epsilon_0}\right)$ & Stochastic & $\left(\epsilon=\epsilon_0,\delta\right)$ & $\left(\epsilon_0,\delta\right)$\\
		\hline\hline
		Central DP (Theorem~\ref{thm:1})& $\tilde{\mathcal{O}}\left(\sqrt{T}+\frac{1}{\epsilon}\right)$ & Free & $\left(\epsilon,0\right)$ & N/A\\
		LDP (Theorem~\ref{thm:2})& $\tilde{\mathcal{O}}\left(\frac{\sqrt{T}}{\epsilon_0}\right)$ & Free & $\left(\epsilon=\epsilon_0,0\right)$ & $\left(\epsilon_0,0\right)$\\
		LDP+shuffling(Theorem~\ref{thm:3})& $\tilde{\mathcal{O}}\left(\sqrt{T}+\frac{1}{\epsilon}\right)$ & Free & $\left(\epsilon,\delta\right)$ & $\left(\epsilon_0=\epsilon T^{1/4},0\right)$\\
		\hline
	\end{tabular}
	\caption{Upper part: known results. Lower part: our results. The $\tilde{\mathcal{O}}$ notation hides the dependencies on the dimension $d$, privacy parameter $\delta$ and $\log$ factors.}~\label{T1}
	\vspace{-.5cm}
\end{table}

\section{Notation and Problem Formulation} \label{sec:problem_formulation}
{\bf Stochastic linear bandits.} In stochastic linear bandits  a learner interacts with clients  over $T$ rounds by taking a sequence of decisions and receiving rewards. In particular, at each round $t\in[T]$, the learner plays an action $a_t$ from a set $\mathcal{A}\subset \mathbb{R}^{d}$ and receives a reward $r_t \in \mathbb{R}$. The reward $r_t$ is a noisy linear function of the action, i.e., $r_t=\langle \theta_{*},a_t\rangle+\eta_t$, where $\langle.\rangle$ denotes inner product, $\eta_t$ is an independent zero-mean noise and $\theta_{*}\in\mathbb{R}^{d}$ is an unknown parameter vector. The goal of the learner is to minimize the total regret over  the $T$ rounds, which  is calculated  as:
\begin{equation}~\label{eqn:regert}
	R_{T}=T\max_{a\in\mathcal{A}}\langle\theta_{*},a\rangle - \sum_{t=1}^{T}\langle \theta_{*},a_t\rangle.
\end{equation} 
The regret captures the difference between the reward for the optimal action and the rewards for the actions chosen by the learner. The basic approach in all algorithms is to play actions that enable the learner to learn $\theta_{*}$ well enough to identify a (near) optimal action.
The best known algorithms (for example, LinUCB~\cite{abbasi2011improved,rusmevichientong2010linearly}) achieve a regret of order $O(\sqrt{T\log T})$, which is the best we can hope for (matches existing lower bounds~\cite{rusmevichientong2010linearly}). 

In this paper, we make the following standard assumptions (see, e.g.,~\cite{abbasi2011improved,shariff2018differentially}).
\begin{assumption} \label{ass:1}
	We consider stochastic linear bandits with:\\
		1. Sub-gaussian noise: $\mathbb{E}[\eta_{t+1}|\mathcal{F}_t]=0$ and $\mathbb{E}[\exp(\lambda \eta_{t+1})|\mathcal{F}_t]\leq \exp(\frac{\lambda^2}{2})\forall \lambda\in \mathbb{R}$, where $\mathcal{F}_t=\sigma(a_1,r_1,...,a_t,r_t)$ is the $\sigma$-field summarizing the information available before round $t$.\\
		2. Bounded actions: $\|a\|_2\leq 1\ \forall a\in \mathcal{A}$.\\
		3. Bounded unknown parameter: $\|\theta_*\|_2\leq 1$.\\
		4. Bounded rewards: $|r_t|\leq 1$.\label{assump:bdd_reward}
\end{assumption}

{\bf Privacy Goal and Measures.}
Our goal in this paper is to achieve the  minimum possible  regret in~\eqref{eqn:regert} while preserving privacy of the rewards $\lbrace r_t\rbrace_{t\in[T]}$ (as discussed in Section~\ref{sec:intro} the rewards can represent sensitive information of the clients). 
To measure privacy, we use the popular central and local differential privacy definitions that we provide for completeness next. For simplicity, we assume that a different client plays each action (e.g., visits a recommended restaurant). 

{\bf Differential Privacy (DP).}
We say that two sequences of rewards $\mathcal{R}=(r_1,\ldots,r_T)$ and $\mathcal{R}'=(r'_1,\ldots,r'_T)$ are neighboring if they  differ in a single reward, i.e., there is a round $t\in[T]$ such that $r_t\neq r'_t$, but $r_j=r'_j$ for all $j\neq t$.
To preserve privacy, we use a randomized mechanism $\mathcal{M}$ designed for stochastic linear bandits, that observes rewards and outputs publicly observable actions.


\begin{definition}\label{def:DP} (\cite{Calibrating_DP06,dwork2014algorithmic}): A randomized mechanism $\mathcal{M}$ for stochastic linear bandits is said to be $\left(\epsilon,\delta\right)$ Differentially Private ($\left(\epsilon,\delta\right)$-DP) if for any two neighboring sequences of rewards $\mathcal{R}=(r_1,\ldots,r_T)$ and $\mathcal{R}'=(r'_1,\ldots,r'_T)$, and any subset of outputs $\mathcal{O}\subset\mathcal{A}^{T}$, $\mathcal{M}$ satisfies:
	\begin{equation}
		\Pr[\mathcal{M}\left(\mathcal{R}\right)\in\mathcal{O}]\leq e^{\epsilon}\Pr[\mathcal{M}\left(\mathcal{R}'\right)\in\mathcal{O}]+\delta.
	\end{equation}
\end{definition} 
When $\delta=0$, we say that the mechanism $\mathcal{M}$ is pure differentially private ($\epsilon$-DP). The DP mechanisms maintain that the distribution on the output of the mechanism does not significantly change  when replacing a single client with reward $r_t$ with another client with reward $r'_t$. Thus, the adversary observing the output of the DP mechanism does not infer the clients rewards.

{\bf Local Differential Privacy (LDP).} If the central learner is untrusted,
we need a local private mechanism $\mathcal{M}$ whose output is all the information available to the central learner.
We denote the range of the output of the local mechanism by $\mathcal{Z}$.
\begin{definition} (\cite{kasiviswanathan2011can})~\label{def:ldp} A randomized mechanism $\mathcal{M}:[-1,1]\to \mathcal{Z}$ is said to be $\left(\epsilon_0,\delta_0\right)$ Local Differentially Private ($\left(\epsilon_0,\delta_0\right)$-LDP) if for any rewards $r_t$ and $r'_t$, and any subset of outputs $\mathcal{O}\subset\mathcal{Z}$, the algorithm $\mathcal{M}$ satisfies:
	\begin{equation}
		\Pr[\mathcal{M}\left(r_t\right)\in\mathcal{O}]\leq e^{\epsilon_0}\Pr[\mathcal{M}\left(r'_t\right)\in\mathcal{O}]+\delta_0.
	\end{equation}
\end{definition}
Similar to the DP definition, we say that  $\mathcal{M}$ is pure locally differentially private ($\epsilon_0$-LDP) when $\delta_0=0$. Observe that the input of the LDP mechanism is a single reward, and hence, each client preserves privacy of her observed reward $r_t$, even if the adversary knows what is the action she plays and observes a function of her reward. 

{\bf System Model.} We consider three different  models for private stochastic linear bandits. In all three cases, our setup is that of a learner, who asks clients to play publicly observable actions, and collects the resulting rewards using a secure communication channel (see Figure~\ref{fig:system_model}). The models differ on whether the learner is a trusted or untrusted server, and whether a shuffler is available or not. A shuffler simply performs a random permutation on its input.  
\\
{\bf 1) Central DP model:}  The learner is a {\bf trusted server} who can collect the clients' rewards and take actions.  Thus, the trusted server can apply a DP mechanism (see Definition~\ref{def:DP}) to preserve the privacy of the collected rewards against any adversary observing the actions of the clients.\\
{\bf 2) LDP model}: The learner is an {\bf untrusted server}. Hence, each client needs to privatize her own reward by applying an LDP mechanism (see Definition~\ref{def:ldp}) before sending it to the untrusted server. The server takes decisions on next actions using the collected privatized rewards.\\
{\bf 3) Shuffled model}: Similar to the LDP model, the learner is an {\bf untrusted  server}. However, we consider that there exists a {\bf trusted shuffler} that collects the LDP responses of the clients and randomly permutes them before passing them to the server, see Figure~\ref{fig:system_model}.

\section{Stochastic Linear Bandits with central DP}\label{sec:central}
In this section we consider the case where the learner is a trusted server. We present an algorithm that offers $\epsilon$-DP (see Definition~\ref{def:DP}) for stochastic linear bandits, with no regret penalty: we achieve the same order regret performance as the best algorithms  that operate under no privacy considerations. 

{\bf Main Idea.} Our algorithm follows the structure of elimination algorithms: it runs in batches, where in each batch $i$ we maintain a ``good set of actions'' $\mathcal{A}_i,$ that almost surely contain the optimal one, and gradually eliminate sub-optimal actions, shrinking the sets  $\mathcal{A}_i$ as $i$ increases.
As is fairly standard in elimination algorithms, in our case as well,
during batch $i$, the learner plays actions in  $\mathcal{A}_i$, calculates an updated estimate $\hat{\theta}_i$ of the unknown parameter vector $\theta_*$, and eliminates from   $\mathcal{A}_i$ actions
if their estimated reward is   $2\gamma_i$ from the estimated reward of the arm that appears to be best, where $\gamma_i$ is the confidence of the reward estimates.

Our new idea, that enables to make our algorithm offer  $\epsilon$-DP, is at a high level as follows.  {\bf If by playing a smaller number of distinct actions we are able to identify the optimal action, we need to overall add a smaller amount of noise to guarantee privacy than if we play a larger number of distinct actions.}
Indeed, if an action $a$ is played for $n_a$ times, the learner, to estimate $\theta_*$, only needs to use the sum of these $n_a$ rewards. To offer $\epsilon$-DP we can perturb this sum by adding independent Laplacian noise (Lap($\frac{1}{\epsilon}$)); clearly, the smaller the number of distinct actions we play, the smaller the overall amount of noise we need to add. Thus our algorithm, at each batch iteration $i$,   plays actions from a carefully selected  subset of $\mathcal{A}_i$, of cardinality as small as possible. The technical question we address is, starting from a continuous action space $\mathcal{A}$, how to select at each batch iteration a small cardinality subset that maintains the ability to identify the optimal action.

We next describe the steps in implementing this idea.
Recall that our actions come from a set $\mathcal{A}\subseteq \mathbb{R}^{d}$, and we assume they are bounded, namely, $\|a\|_2\leq 1,\ \forall a\in \mathcal{A}$ (see Assumptions~\ref{ass:1} in Section~\ref{sec:problem_formulation}).\\
${\bf 1.}$  Our first step is to {\bf reduce the continuous action space to a discrete action space problem}. To do so,
we finely discretize $\mathcal{A}$ to create  what we call a $\zeta$-net,  a discrete set of actions  $\mathcal{N}_\zeta \subseteq \mathcal{A}$ such that distances are approximately preserved. Namely, for any $a\in \mathcal{A}$, there is some $a'\in \mathcal{N}_\zeta$ with $\|a'-a\|_2\leq \zeta$. Lemma~\ref{lem:net}, proved in~\cite[Cor. 4.2.13]{vershynin2018high}, states that we can always find such a discrete set with cardinality at most $(\frac{3}{\zeta})^{d}+d$. As a result, all the ``good sets'' $\mathcal{A}_i$ will also be discrete.  
\begin{lemma}( $\zeta$-net for $\mathcal{A}$~\cite{vershynin2018high})\label{lem:net}	For any set $\mathcal{A}\subseteq \{x\in \mathbb{R}^d| \|x\|_2\leq 1\}$ that spans $\mathbb{R}^d$, there is a set $\mathcal{N}_\zeta \subseteq \mathcal{A}$ ($zeta$-net) with cardinality at most $(\frac{3}{\zeta})^{d}+d$ such that $\mathcal{N}_\zeta$ spans $\mathbb{R}^d$, and for any $a\in \mathcal{A}$, there is some $a'\in \mathcal{N}_\zeta$ with $\|a'-a\|_2\leq \zeta$.
\end{lemma}
${\bf 2.}$
We introduce the use of a {\bf core set} $\mathcal{C}_i$, a subset of the actions of the set of ``good actions'' $\mathcal{A}_i$. During batch $i$, {\bf the learner only  plays actions  in $\mathcal{C}_i$, each with some probability $\pi_i(a)$}. Lemma~\ref{lem:core}, proved in[Ch.21]~\cite{lattimore2020bandit}, states that if $\mathcal{A}_i$ spans some space $\mathbf{R}^k$, we can find a core set of size  at most $Bk$ (with $B$ a constant) and an associated probability distribution $\pi$, so that, playing actions only  from $C_i$ enables to calculate a good estimate of $\langle a,\theta_*\rangle$ for each $a\in\mathcal{A}_i$ .
\begin{lemma}(Core set for $\mathcal{A}$~\cite{lattimore2020bandit})\label{lem:core}
	For any finite set of actions $\mathcal{A}\subset \{x\in \mathbf{R}^d|\|x\|_2\leq 1\}$ that spans $\mathbb{R}^d$, there is a subset $\mathcal{C}$ of size at most $Bd$ that spans $\mathbb{R}^d$, where $B$ is a constant, and a distribution $\pi$ on  $\mathcal{C}$ such that for any $a\in \mathcal{A}$
	\begin{equation}\label{eq:corelemm}
		a^\top \left(\sum_{\alpha \in \mathcal{C}} \pi(\alpha) \alpha \alpha^\top\right)^{-1} a\leq 2d.
	\end{equation}
	Moreover, $\mathcal{C}$ and $\pi$ can be found in polynomial time in $d$.
\end{lemma}
{\bf 3.} To preserve the privacy of rewards, we {\bf perturb the sum rewards of each action by adding Laplace noise}. Adding noise affects the confidence of the reward estimates $\gamma$ ({step~\ref{thr} in Algorithm~\ref{alg:cen} shows that $\gamma$ increases as $\epsilon$ decreases}), and thus delays the elimination of bad actions and increases the regret by an additive term of $\tilde{O}(\frac{1}{\epsilon})$. Replacing a possibly large set $\mathcal{A}_i$  with the smaller core set $\mathcal{C}_i$ effectively decreases the cumulative noise affecting the estimate of $\theta_*$.


{\bf Remark.}
{ The computation of $\mathcal{C},\pi$ can be formulated as a convex optimization problem with many efficient approximation algorithms available. One example is the Franke-Wolfe algorithm~\cite{frank1956algorithm,lattimore2020bandit} that starts with an initialization of $\pi$, $\pi_0$ and updates it according to 
\begin{equation}
	\pi_{i+1}(a)=(1-\nu_i)\pi_{i}(a)+\nu_i \mathbf{1}(a=\arg\max_{\alpha \in \mathcal{A}}\|\alpha\|^2_{V(\pi_i)^{-1}}),
\end{equation}
where $V(\pi_i)=\sum_{\alpha \in \mathcal{A}} \pi_i(\alpha) \alpha \alpha^\top $ and $\nu_i$ is a step size. If $\pi_0$ is chosen to be the uniform distribution over $\mathcal{A}$, then in $O(d\log \log |\mathcal{A}|)$ iterations we can find a $\pi$, and  $\mathcal{C}=\{a\in \mathcal{A}| \pi(a)\neq 0\}$ that satisfy \eqref{eq:corelemm}. Using a more sophisticated initialization, the dependence on $|\mathcal{A}|$ in the number of iterations can be eliminated entirely and the core set is guaranteed to have size of $O(d)$. 
}

{\bf Algorithm  Pseudo-Code.}
Our algorithm pseudo-code in Algorithm~\ref{alg:cen}, starts by initializing the good action set $\mathcal{A}_1$ to be an $\frac{1}{T}$-net of $\mathcal{A}$ according to Lemma~\ref{lem:net}. Then, the algorithm operates in batches that grow exponentially in length, where the length of batch $i$ is approximately $q^i$ and $q=(2T)^{1/\log T}$\footnote{We note that $e\leq q\leq e^2$.}. In each batch $i$, we construct the core set $C_i$ and the associated distribution $\pi_i$ $\mathcal{A}_i$ as per Lemma~\ref{lem:core}. Each action in $\mathcal{C}_i$ is pulled $n_{ia}=\ceil{\pi(a)q^i}$ times, where the length of batch $i$ is $n_i=\sum_{a\in \mathcal{C}_i}n_{ia}$. To preserve privacy, the sum of the rewards of each action is perturbed with Lap$(\frac{1}{\epsilon})$ noise. The learner uses these privatized sum rewards to compute the least squares estimate of $\theta_{*}$, $\hat{\theta}_i$.  At the end of batch $i$ the learner eliminates from $\mathcal{A}_i$ the  actions with estimated mean reward, $\langle a,\hat{\theta}_i \rangle$, that fail to be within $2\gamma_i$ from the action that appears to be best,   where $\gamma_i$ is our confidence in the mean estimates.
After the iteration $i=\log T -1$ is completed, the learner simply plays the action that appears to be best.

\begin{algorithm} 
\caption{$\epsilon$-DP algorithm for stochastic linear bandits: { central model}}\label{alg:cen}
\begin{algorithmic}[1]
	\State Input: set of actions $\mathcal{A}$,  time horizon $T$, and privacy parameter $\epsilon$.
	\State Let $\mathcal{A}_1$ be a $\zeta$-net for $\mathcal{A}$ as in Lemma~\ref{lem:net}, with $\zeta=\frac{1}{T}$.
	\State $q \gets (2T)^{1/\log T}$.
	\For{$i = 1:\log(T)-1$} 
	\State $\gamma_i\gets \sqrt{\frac{4d}{q^{i}}\log\left(4|\mathcal{A}_i|T^2\right)}+\frac{2Bd^2+2d\log\left(4|\mathcal{A}_i|T^2\right)}{\epsilon q^{i}}$.\label{thr}
	\State For $\mathcal{A}_i\subseteq\mathbf{R}^m$, $m\leq d$,
	let  $\mathcal{C}_i$ be a core set of size at most $Bm$ as in Lemma~\ref{lem:core} and $\pi_i$ the associated distribution.
	\State Pull each action $a\in \mathcal{C}_i$, $n_{ia}=\ceil{{\pi_i(a)q^i}}$ times to get rewards $r_{ia}^{(1)},...,r_{ia}^{(n_{ia})}$.
	\State~\label{step:cental_privacy} $\Bar{r}_{ia}\gets \sum_{k=1}^{n_{ia}}r_{ia}^{(k)},\ \hat{r}_{ia}\gets \Bar{r}_{ia}+z_{ia} \ \forall a\in \mathcal{C}_i$, where $z_{ia}$ is an independent noise that follows $\mathsf{Lap}(\frac{1}{\epsilon})$.
	\State $V\gets \sum_{a\in \mathcal{C}_i}n_{ia}aa^\top,\ \hat{\theta}_i\gets V^{-1}\sum_{a\in \mathcal{C}_i}\hat{r}_{ia}a$.
	\State $\mathcal{A}_{i+1}\gets \{a\in \mathcal{A}_i| \langle a,\hat{\theta}_i\rangle\geq \max_{\alpha \in \mathcal{A}}\langle \alpha,\hat{\theta}_i \rangle-2\gamma_i\}$
	\EndFor
	\State Play action $\arg\max_{\alpha\in \mathcal{A}_{\log(T)-1}}\langle \alpha,\hat{\theta}_{\log(T)-1} \rangle$ for the remaining time.		
\end{algorithmic}
\end{algorithm}

{\bf Algorithm Performance.} We next prove that Algorithm~\ref{alg:cen} is $\epsilon$-DP and provide a bound on its regret.
\begin{theorem}\label{thm:1}
Algorithm~\ref{alg:cen} is $\epsilon$-differentially private. Moreover, it achieves a regret
\begin{equation}\label{eq:reg_prob}
	R_T\leq C\left(\sqrt{T\log T}+\frac{\log^2 T}{\epsilon}\right),
\end{equation}
with probability at least $1-\frac{1}{T}$, where $C$ is a constant that does not depend on $\epsilon,T$.
\end{theorem}
{\bf Proof Outline.} The privacy result follows from the Laplace mechanism \cite{dwork2014algorithmic}. To bound the regret, we first argue that with probability at least $1-\frac{1}{T}$,  and for all $i$ and all $a\in \mathcal{A}_i$, we have that $|\langle a,\hat{\theta}_i \rangle-\langle a,\hat{\theta}_\star \rangle|\leq \gamma_i$. Conditioned on this event, an action with gap $\Delta_a$ is eliminated when, or before, $\gamma_i< \Delta_a/2$. Hence, all actions in batch $i$ have gap that is at most $4\gamma_i$. The regret bound follows by summing $4\gamma_i n_i$ for all batches.
The  complete proof is provided in Appendix~\ref{sec:analysis_central}. \hfill{$\square$}

\begin{remark}
We note that the high probability bound in Theorem~\ref{thm:1} implies a bound in expectation
\begin{equation}
	\mathbb{E}[R_T]\leq C\left(\sqrt{T\log T}+\frac{\log^2 T}{\epsilon}\right).
\end{equation} 
This is because the regret is trivially $O(T)$ and the algorithm fails with probability $\frac{1}{T}$, which overall contributes $O(1)$  to the expectation.
\end{remark}
\begin{remark}
The regret in Theorem~\ref{thm:1} is optimal up to $\log T$ factor; a lower bound of $O(\sqrt{T})$ is proven in \cite{rusmevichientong2010linearly} for the non-private case, while a lower bound of $\frac{\log T}{\epsilon}$ is shown in \cite{shariff2018differentially} for the private case.
\end{remark}
\begin{remark}
We observe that the privacy parameter $\epsilon$ is typically $\approx 1$. In this case, the dominating term in the regret in \eqref{eq:reg_prob} is $O(\sqrt{T\log T})$ which matches the regret of the best known algorithm for the non-private case (see LinUCB in \cite{rusmevichientong2010linearly, abbasi2011improved}), and hence, we get privacy for free.
\end{remark}

\section{Stochastic Linear Bandits with LDP}\label{sec:LDP}
In this section, the learner is an untrusted server, and thus we design a linear bandit algorithm  (Algorithm~\ref{alg:local}) that operates under LDP constraints.

{\bf Main Idea.} 
As in Algorithm~\ref{alg:cen}, we here also utilize a core set of actions; the difference is that,
since the server is untrusted,
each client privatizes  her own reward before providing it to the  server.
Our algorithm offers an alternative approach to~\cite{han2021generalized} that achieves the same regret, while using operation in batches, which may in some applications be more implementation-friendly, and also forms a foundation for the  Algorithm~\ref{alg:shuffle} we discuss in the next section.


\begin{algorithm} 
	\caption{$\epsilon_0$-LDP algorithm for stochastic linear bandits: local model}\label{alg:local}
	\begin{algorithmic}[1]
		\State Input: set of actions $\mathcal{A}$,  time horizon $T$, and privacy parameter  $\epsilon_0$.
		\State Let $\mathcal{A}_1$ be a $\zeta$-net for $\mathcal{A}$ as in Lemma~\ref{lem:net}, with $\zeta=\frac{1}{T}$.
		\State $q \gets (2T)^{1/\log T}$.
		\For{$i = 1:\log(T)-1$} 
		\State \textbf{Client side}:
		\State \quad Receive action $a$ from the server. Play action $a$ and receive a reward $r$.
		\State~\label{step:local_privacy}\quad Send $\hat{r}=r+\mathsf{Lap}(\frac{1}{\epsilon_0})$.
		\State \textbf{Server side}:
		\State \quad Let $\mathcal{C}_i$ be a core set for $\mathcal{A}_i$ as in Lemma~\ref{lem:core} with distribution $\pi_i$, and $n_{ia}=\ceil{{\pi_i(a)q^i}}$.
		\State \quad Send each action $a\in \mathcal{C}_i$ to a set of $n_{ia}$ clients to get rewards $\hat{r}_{ia}^{(1)},...,\hat{r}_{ia}^{(n_{ia})}$.
		\State \quad $n_i\gets \sum_{a\in\mathcal{C}_i}n_{ia}$.
		\State~\label{step:local_threshold} \quad $\gamma_i\gets \sqrt{\frac{4d}{q^{i}}\log\left(4|\mathcal{A}_i|T^2\right)}+\frac{2d}{q^i\epsilon_0}\sqrt{n_i\log (4|\mathcal{A}_i|T^2)}$.\label{thr_local}
		\State \quad $\hat{r}_{ia}\gets \sum_{k=1}^{n_j}\hat{r}_{ia}^{(1)} \ \forall a\in \mathcal{C}_i$.
		\State \quad $V\gets \sum_{a\in \mathcal{C}_i}n_{ia}aa^\top,\ \hat{\theta}_i\gets V^{-1}\sum_{a\in \mathcal{C}_i}\hat{r}_{ia}a$.
		\State \quad $\mathcal{A}_{i+1}\gets \{a\in \mathcal{A}_i| \langle a,\hat{\theta}_i\rangle\geq \max_{\alpha \in \mathcal{A_i}}\langle \alpha,\hat{\theta}_i \rangle-2\gamma_i\}$. 
		\EndFor
		\State Play action $\arg\max_{\alpha\in \mathcal{A}_{\log(T)-1}}\langle \alpha,\hat{\theta}_{\log(T)-1} \rangle$ for the remaining time.	
	\end{algorithmic}
\end{algorithm}

{\bf Algorithm Pseudocode.} 
	Algorithm~\ref{alg:local}  operates like Algorithm~\ref{alg:cen}, except for the addition of $\mathsf{Lap}(1/\epsilon_0)$ noise for each reward individually as opposed to adding $\mathsf{Lap}(1/\epsilon)$ to the sum of the rewards of each arm in the central model. The value of $\gamma_i$ is adjusted to account for this change.
	{\bf Algorithm Performance.} The following  Theorem~\ref{thm:2} presents the privacy-regret tradeoffs of the LDP stochastic bandits Algorithm~\ref{alg:local}.  
	The  proof is deferred to Appendix~\ref{sec:analysis_local} and follows the same main steps as the proof of Theorem~\ref{thm:1}, but with the modified values of $\gamma_i$.
	\begin{theorem}\label{thm:2}
		Algorithm~\ref{alg:local} is $\epsilon_0$-LDP. Moreover, it achieves a regret
		\begin{equation}\label{eq:reg_prob_local}
			R_T\leq C(1+\frac{1}{\epsilon_0})\left(d\sqrt{dT\log T}\right),
		\end{equation}
		with probability at least $1-\frac{1}{T}$, where $C$ is a constant that does not depend on $\epsilon_0$ and $T$.
	\end{theorem}
	\begin{remark} Since the regret is trivially bounded by $O(T)$ when Algorithm~\ref{alg:local} fails, which happens with probability $\frac{1}{T}$, we can upper bound the expected regret as
		\begin{equation}
			\mathbb{E}[R_T]\leq C(1+\frac{1}{\epsilon_0})\left(d\sqrt{dT\log T}\right).
		\end{equation}
	\end{remark}
	\begin{remark}\label{rem:local_loweps}
		When $\epsilon_0>1$, the regret $R_T$ would be $\mathcal{O}\left(\sqrt{T}\log(T)\right)$ that matches the non-private case. However, the constants of the regret convergence are larger than that of the non-private case. 
	\end{remark}
	\begin{remark}(Comparison to the central $(\epsilon,\delta)$-DP model.) Observe that when $\epsilon_0<1$, the dominating term in the regret bound is $R_T=\mathcal{O}\left(\frac{T\log(T)}{\epsilon_0}\right)$. In other words, we obtain the regret of the non-private case divided by the LDP parameter $\epsilon_0$. In contrast, the central DP parameter $\epsilon$ appears as an additive term in the regret of the central model. This difference is because, in the local model noise is added on every reward,
		while in the central model directly on the reward aggregates; thus the noise variance of the aggregate rewards and the confidence parameter $\gamma_i$ increases in the local model. 
		%
		In the high privacy regimes; for example, assume that  $\epsilon_0=\mathcal{O}\left(\frac{1}{T^{\alpha}}\right)$ for some $0<\alpha\leq \frac{1}{2}$, we  get a regret $R_T$ of order $\mathcal{O}\left(T^{\frac{1}{2}+\alpha}\right)$ that becomes linear function of $T$ as $\epsilon_0\to \frac{1}{\sqrt{T}}$.
	\end{remark}

\section{Stochastic Linear Bandits in the Shuffled Model}\label{sec:shuffled}
\begin{algorithm}
	\caption{DP algorithm for stochastic linear bandits: shuffled model}\label{alg:shuffle}
	\begin{algorithmic}[1]
		\State Input: set of actions $\mathcal{A}$, time horizon $T$, and privacy parameters $(\epsilon,\delta)$.
		\State Let $\mathcal{A}_1$ be a $\zeta$-net for $\mathcal{A}$ as in Lemma~\ref{lem:net},  with $\zeta=\frac{1}{T}$.
		\State $q \gets (2T)^{1/\log T}$.
		\For{$i = 1:\log(T)-1$} 
		\State \textbf{Client side}:
		\State \quad Receive action $a$ and the value $n_i$ from the shuffler.
		\State \quad Play action $a$ and receive a reward $r$.
		\State \quad $\epsilon_0^{(i)}\gets f_{n_i,\delta}^{-1}(\epsilon)$
		\State \quad Send $\hat{r}=r+\mathsf{Lap}(\frac{1}{\epsilon_0^{(i)}})$ to the shuffler.\label{noise}
		\State \textbf{Shuffler}:
		\State \quad Let $\mathcal{C}_i$ be a core set  for $\mathcal{A}_i$ as in Lemma~\ref{lem:core} with distribution $\pi_i$.
		\State \quad Let $n_{ia}=\ceil{{\pi_i(a)q^i}}, n_i\gets \sum_{a\in\mathcal{C}_i}n_{ia}$.
		\State \quad Let $\mathcal{A}_{\mathcal{C}_i}=\cup_{a\in \mathcal{C}_i} \{a\}_{l=1}^{n_{ia}}$ be a set of $n_i$ actions where action $a\in \mathcal{C}_i$ is repeated $n_{ia}$ times.
		\State \quad Let $a_1,...,a_{n_i}$ be an enumeration of $\mathcal{A}_{\mathcal{C}_i}$.
		\State \quad Send action $a_{\pi(j)}$ {and the value $n_i$} to client $j,\ j=1,...,n_i$, where $\pi$ is a random permutation of $1,...,n_i$. 
		\State \quad Receive the action-reward pairs $\{(a_1,\hat{r}_{ia_1}),...,(a_{n_i},\hat{r}_{ia_{n_i}})\}$, and send them to the server.
		\State \textbf{Server side}:		
		\State \quad Receive the action-reward pairs from the shuffler.
		\State \quad $\gamma_i\gets \sqrt{\frac{4d}{q^{i}}\log\left(4|\mathcal{A}_i|T^2\right)}+\frac{2d}{q^i\epsilon_0^{(i)}}\sqrt{n_i\log (4|\mathcal{A}_i|T^2)}$.
		\State \quad $\hat{r}_{ia}\gets \sum_{k=1}^{n_j}\hat{r}_{ia}^{(1)} \ \forall a\in \mathcal{C}_i$.
		\State \quad $V\gets \sum_{a\in \mathcal{C}_i}n_{ia}aa^\top,\ \hat{\theta}_i\gets V^{-1}\sum_{a\in \mathcal{C}_i}\hat{r}_{ia}a$.
		\State \quad $\mathcal{A}_{i+1}\gets \{a\in \mathcal{A}_i| \langle a,\hat{\theta}_i\rangle\geq \max_{\alpha \in \mathcal{A}}\langle \alpha,\hat{\theta}_i \rangle-2\gamma_i\}$. 
		\EndFor
		\State Play action $\arg\max_{\alpha\in \mathcal{A}_{\log(T)-1}}\langle \alpha,\hat{\theta}_{\log(T)-1} \rangle$ for the remaining time.
		
	\end{algorithmic}
\end{algorithm}
In this section, we consider the case of an untrusted server and a trusted shuffler. We propose Algorithm~\ref{alg:shuffle} that (almost) achieves the same order regret as the best non-private algorithms. 

{\bf Main idea.} To use shuffling, we need to use an algorithm that operates over batches of actions, so as to be able to shuffle them. 
The use of a core set is critical to enable a selection of actions that lead to a good estimate for $\theta\star$. For  example, if the original set $\mathcal{A}$ contains a large number of actions along one direction in the space, but only a few actions along other directions, then pulling each action in $\mathcal{A}$ once will not result in a good estimate of $\theta_\star$. Use of the core set and the associated distribution $\pi$ will balance such assymetries and enable to explore  multiple directions of the space for a sufficient number of times to acquire a good estimate of $\theta_\star$.

Accordingly, we follow the same approach as in Algorithm~\ref{alg:local} with two changes: we use a shuffler (in a manner tailored to bandits) to realize privacy amplification gains, and we adjust the amount of Laplace noise we add in each batch, depending on the batch size. 

We use the trusted shuffler as follows.  The actions to be played in the $i$th batch are shuffled by the trusted shuffler at the beginning of the batch. The shuffler  asks clients to play actions in the shuffled order. Then, at the end of the batch, the shuffler reverses the shuffling operation, associates every action with its observed LDP reward, and conveys it to the untrusted learner.\footnote{We assume  that the server cannot directly observe which action is played by which client, for instance due to geographical separation.}

We adjust the amount of added Laplace noise per batch as follows. To offer privacy guarantees, we want to add noise to the rewards so that the output of the shuffler is $\left(\epsilon,\delta\right)$-DP for each batch $i\in[\log(T)]$. This implies that the entire algorithm will be  $\left(\epsilon,\delta\right)$-DP, since we assume that each each client contributes at only one of the batches. The privacy amplification of the shuffling depends on the size of the batch (see e.g.~\cite[Theorem~$1$]{feldman2022hiding}); thus the larger the batch size, the less noise needs to be added to the rewards of the clients.
To ensure that the output of  batch $i$ is $\left(\epsilon,\delta\right)$-DP, it is suffcient to add to each reward noise $\mathsf{Lap}(\frac{1}{\epsilon_0^{(i)}})$, where $\epsilon_0^{(i)}\gets f_{n_i,\delta}^{-1}(\epsilon)$, and $n_i$ is the size of batch $i$. The function
$f_{n,\delta}:\mathbb{R}^+\to \mathbb{R}^+$  captures privacy amplification via shuffling \cite{feldman2022hiding} and is defined as follows
\begin{equation}
	f_{n,\delta}(\epsilon_0) = \log\left(1+\frac{e^{\epsilon_0}-1}{e^{\epsilon_0}+1}\left(\frac{8\sqrt{e^{\epsilon_0}\log(4/\delta)}}{\sqrt{n}}+\frac{8e^{\epsilon_0}}{n}\right)\right).
\end{equation}
Since the noise added to the rewards varies for each batch $i$, we modify the confidence bounds, $\gamma_i$, to reflect this.
The pseudo-code is provided in Algorithm~\ref{alg:shuffle}.




{\bf Algorithm Performance.}
The following theorem proves that Algorithm~\ref{alg:shuffle} is $(\epsilon,\delta)$-DP and provides an upper bound on its regret that matches the {information theoretic lower bound} for $\epsilon = \tilde{O}(\frac{1}{\sqrt{T}})$.

\begin{theorem}\label{thm:3}
	Algorithm~\ref{alg:shuffle} is $(\epsilon,\delta)$-differentially private.	Moreover, for $\epsilon=O(\sqrt{\frac{\log(1/\delta)}{T}})$ it achieves a regret
	\begin{equation}\label{eq:reg_prob}
		R_T\leq C\left(\sqrt{T\log T}+\frac{\sqrt{\log(1/\delta)}\log^{3/2} T}{\epsilon}\right),
	\end{equation}
	with probability at least $1-\frac{1}{T}$, where $C$ is a constant that does not depend on $\epsilon$ and $T$.
\end{theorem}
{\bf Proof Outline.}
The proof of Theorem~\ref{thm:3} is deferred to Appendix~\ref{sec:analysis_shuffled}. {The privacy guarantee is proved by reducing the scheme to one that shuffles the rewards but does not shuffle the corresponding actions and using results from \cite{feldman2022hiding}. The regret analysis follows similar ideas as in Theorem~\ref{thm:1} and Theorem~\ref{thm:2}.}

\begin{remark}
	Note that if we had the shuffler to simply permute the collected rewards of the clients (and not the actions) we would get no privacy gains in some cases. 
	For example, consider the case where all actions to be pulled in a batch are unique and the {action pulled by each client} is known to the central learner (e.g., for MAB algorithms where the policy is a deterministic function of the history), then the learner can undo the shuffling using the action associated with each shuffled reward.
\end{remark}
\begin{remark} Algorithm~\ref{alg:shuffle} almost achieves the same order regret as the best non-private algorithms.
	Indeed,  Theorem~\ref{thm:3} proves that Algorithm~\ref{alg:shuffle}  achieves a regret that matches the regret of the central DP Algorithm~\ref{alg:cen} for the high privacy regimes $\epsilon=O(\sqrt{\log(1/\delta)/T})$. For the low privacy regime $\epsilon>1$, the shuffling does not offer privacy gains, $\epsilon_{0}^{(i)}\approx \epsilon$ for all $i\in[\log(T)]$ and the regret of Algorithm~\ref{alg:shuffle} is similar to the regret of Algorithm~\ref{alg:local} of the local DP model. However, for the low privacy regime the local DP model also achieves the same regret as  non-private algorithms up to constant factors (see Remark~\ref{rem:local_loweps}). Hence in both cases, Algorithm~\ref{alg:shuffle} achieves the same order regret as Algorithm~\ref{alg:cen} which {almost matches} the regret of non-private algorithms.
\end{remark}
\begin{remark} 
	Algorithm's~\ref{alg:shuffle} improved regret performance over Algorithm~\ref{alg:local} is thanks to the smaller amount of noise added to rewards. In particular, the noise added in Step~\ref{noise} of Algorithm~\ref{alg:shuffle} has variance $\frac{2}{{\epsilon_0^{(i)}}^2}\approx \frac{2}{n_i \epsilon^2}$ for small $\epsilon$.
\end{remark}

\section{Numerical Results}\label{sec:numerics}
\begin{figure}[!t]
	\centering
	\includegraphics[scale=0.8]{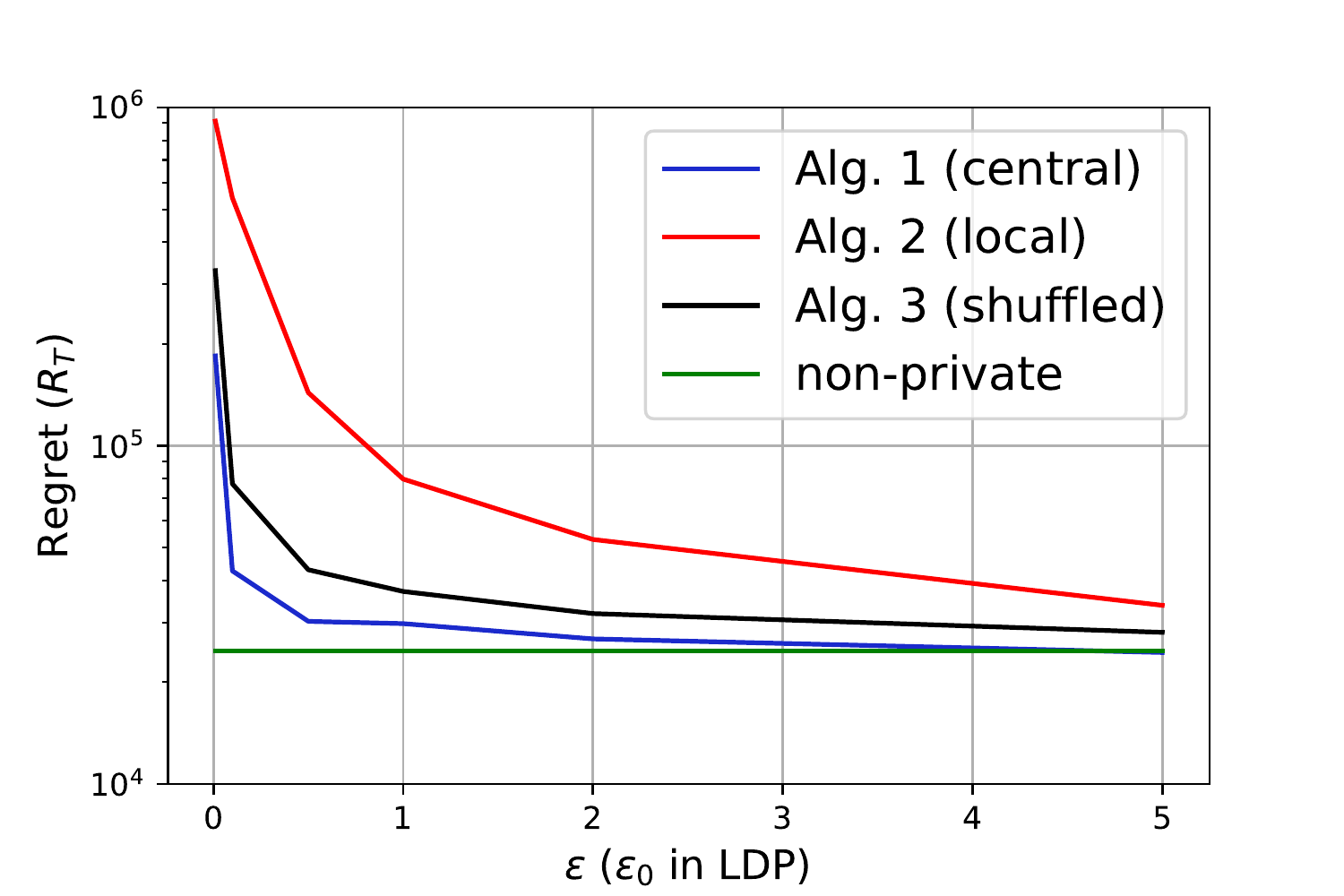}
	\captionof{figure}{Regret-privacy trade-offs  for stochastic linear bandits algorithms.}
	\label{fig:regret}
\end{figure}
We here present indicative results on the performance of our proposed Algorithms~\ref{alg:cen},~\ref{alg:local} and~\ref{alg:shuffle}; additional details and numerical evaluation plots are provided in Appendix~\ref{app:numerics}.  We consider synthetic data generated as follows. The set of actions $\mathcal{A}$ contains $K=10$ actions, where each action $a\in\mathcal{A}$ is a $d=2$-dimensional vector. The actions $a\in\mathcal{A}$ and the optimal parameter $\theta_{*}$ are generated uniformly at random from the unit ball $\mathcal{S}^{d-1}=\{x\in \mathbb{R}^d: ||x||_2= 1\}$ (a similar method is considered in~\cite{han2021generalized}).  Figure~\ref{fig:regret}  plots the total regret $R_T$ over an horizon $T=10^{6}$ as a function of the privacy budget ($\epsilon$ or $\epsilon_0$).  Figure~\ref{fig:regret} shows that the regret achieved by all three algorithms, Algorithm~\ref{alg:cen} (central model),  Algorithm~\ref{alg:local}  (local model), and Algorithm~\ref{alg:shuffle} (shuffled model) converges to the regret of non-private stochastic linear bandit algorithms~\cite[Ch.~$22$]{lattimore2020bandit} as $\epsilon\to\infty$ ($\epsilon_0\to \infty$), albeit at different rates. As predicted from the theoretical analysis, Algorithms~\ref{alg:cen} (central) and~\ref{alg:shuffle} (shuffled) offer privacy (almost) for free, closely following the non-private regret.

\appendix

\section{Regret and Privacy Analysis of The Central DP Model (Proof of Theorem~\ref{thm:1})}\label{sec:analysis_central}
\subsection{Privacy Analysis}
We first show that Algorithm~\ref{alg:cen} is $\epsilon$-DP. Let $\Bar{r}_i=[\Bar{r}_{ia_1},...,\Bar{r}_{ia_{|\mathcal{C}_i|}}],\ \hat{r}_i=[\hat{r}_{ia_1},...,\hat{r}_{ia_{|\mathcal{C}_i|}}]=\Bar{r}_i+z_i, z_i=[z_{ia_1},...,z_{ia_{|\mathcal{C}_i|}}]$, where $a_1,...,a_{|\mathcal{C}_i|}$ is an enumeration of the elements of $\mathcal{C}_i$. We construct the concatenated reward vector denoted by $\Bar{r}=[\Bar{r}_1,...,\Bar{r}_{\log(T)-1}]$, and let $\hat{r}=[\hat{r}_1,...,\hat{r}_{\log(T)-1}]=\Bar{r}+z, z=[z_1,...,z_{\log(T)-1}]$. 

Now consider two neighboring sequence of rewards $\mathcal{R}, \mathcal{R}'$, that only differ in $r_k, r'_k$, with corresponding concatenated reward vectors $\Bar{r}, \Bar{r}'$. We notice that each reward in $\mathcal{R}$ appears once in $\Bar{r}$, and similarly for $\mathcal{R}',\Bar{r}'$. Thus, we get:
\begin{equation}
	\|\Bar{r}-\Bar{r}'\|_1\leq \max_{r_k, r'_k} |r_k-r'_k|\leq 1,
\end{equation}
where the last inequality follows from Assumption~\ref{assump:bdd_reward} with bounded rewards $|r_k|\leq 1$. Then, from~\cite[Theorem~$3.6$]{Calibrating_DP06}, $\hat{r}$ is $\epsilon$-DP. We notice that the output of Algorithm~\ref{alg:cen} depends on $r_1,...,r_T$ only through $\hat{r}$. Hence, by post processing, Algorithm~\ref{alg:cen} is $\epsilon$-DP.

\subsection{Regret Analysis}~\label{subsec:central_regret}
We next prove the regret bound of Algorithm~\ref{alg:cen} for stochastic linear bandits. 

Our analysis follows the known confidence bound technique in~\cite{auer1995gambling} by designing confidence intervals (in step~\ref{thr}) that take into consideration the privacy effect. 

Let $K = \left({3T}\right)^{d}$ be the size of the $\frac{1}{T}$-net set $\mathcal{N}_{1/T}$ from Lemma~\ref{lem:net}. We first bound the following regret:
\begin{equation}\label{eq:reg_net}
	\tilde{R}_{T} = T\max_{a\in\mathcal{N}_{1/T}}\langle a,\theta_{*}\rangle -\sum_{t=1}^{T}\langle a_t,\theta_{*}\rangle, 
\end{equation} 
where $a_1,a_2,\ldots,a_{T} \in \mathcal{N}_{1/T}$. We then bound the regret $R_{T}$ by showing that we only loose a constant term when we choose actions from $\mathcal{N}_{1/T}$ instead of the bigger set $\mathcal{A}$. 

We start with a set of actions $\mathcal{A}_0=\mathcal{N}_{1/T}$ with cardinality $|\mathcal{A}_0|=K$. Furthermore, we have $|\mathcal{A}_i|\leq |\mathcal{A}_{i-1}|$, and hence, we get $|\mathcal{A}_i|\leq K$ for all $i\in[\log(T)]$.

For given batch $i\in[\log(T)]$, let $\mathcal{C}_i$ be the core set of $\mathcal{A}_i$ that has at most $Bd$ actions. At the $i$th batch, each action $a\in\mathcal{C}_i$ is picked $n_{ia}$ times, where $n_{ia}=\ceil{\pi_i(a)q^{i}}$. Let $\mathcal{G}$ be the good event $\left\{\left|\langle a,\hat{\theta}_i-\theta_*\rangle\right|< \gamma_i\ \forall i\in [\log T]\ \forall a\in \mathcal{A}_i\right\}$. Lemma~\ref{lem:concent} shows that the event $\mathcal{G}$ holds with probability at least $1-\frac{1}{T}$. In the remaining part of the proof, we condition on the event $\mathcal{G}$.

We first show that the best action $a_* = \arg\max_{a\in\mathcal{N}_{1/T}}\langle a,\theta_{*}\rangle$ will not be eliminated at any batch $i\in [\log T]$; this is because the elimination criterion will not hold for the optimal action $a_*$:
\begin{equation}
	\langle a,\hat{\theta}_i\rangle - \langle a_*,\hat{\theta}_i\rangle< (\langle a,{\theta}_*\rangle+\gamma_i) - (\langle a_*,{\theta}_*\rangle-\gamma_i)\leq 2\gamma_i\qquad \forall a\in \mathcal{A}_i\ \forall i\in [\log T].
\end{equation}
For each sub-optimal action $a\in \mathcal{A}_0$ with $\Delta_a=\langle a_*-a,\theta_* \rangle$, let $i$ be the smallest integer for which $\gamma_i<\frac{\Delta_a}{4}$. From the triangle inequality, we get that
\begin{equation}
	\langle a_*,\hat{\theta_i} \rangle - \langle a,\hat{\theta_i} \rangle\geq (\langle a_*,\hat{\theta_*} \rangle -\gamma_i) - (\langle a,\hat{\theta_i}\rangle +\gamma_i)=\Delta_a -2\gamma_i>2\gamma_i.
\end{equation}
This implies that $a$ will be eliminated before the beginning of batch $i+1$. Hence, each action $a\in \mathcal{A}_{i+1}$ at batch $i+1$ has a gap at most $4\gamma_i$. Let $n_i=\sum_{a\in \mathcal{C}_i}n_{ia}\leq Bd+q^i$ denote the total number of rounds at the $i$-th batch. Note that the number of batches is upper bounded by $\log T$ since $\sum_{i=1}^{\log T}q^i\geq T$. When $q^i<Bd$, the regret can be bounded by $2Bd$, and when $q^i\geq Bd$, we bound $n_i\leq 2q^i$. Thus, there is universal constants $C',C$ such that the total regret in \eqref{eq:reg_net} can be bounded as
\begin{align}
	\tilde{R}_T &\leq 2Bd\log(T) +\sum_{i=1}^{\log T}4n_i \gamma_{i-1} \\
	&\leq 2Bd\log(T) +\sum_{i=1}^{\log T} 8 q^i \left(\sqrt{\frac{4d}{q^{i-1}}\log\left(4KT^2\right)}+\frac{2Bd^2+2d\log\left(4KT^2\right)}{\epsilon q^{i-1}}\right)\nonumber \\
	&\leq C'\left(d\log (T)+ d\sqrt{\log T}\sum_{i=1}^{\log T} q^{(i-1)/2}+\frac{d^2\log^2 T}{\epsilon}\right)q\nonumber \\
	&\stackrel{(a)}{\leq} C'q\left(d\log (T)+ d\sqrt{\log T}q^{\log T/2}+\frac{d^2\log^2 T}{\epsilon}\right)\nonumber \\
	&\stackrel{(b)}{\leq} C'q\left(d\log (T)+ d\sqrt{T\log T}+\frac{d^2\log^2 T}{\epsilon}\right)\nonumber \\
	&\stackrel{(c)}{\leq} C\left(d\sqrt{T\log T}+\frac{d^2\log^2 T}{\epsilon}\right),
\end{align}
where step $(a)$ follows from the sum of a geometric series and $q>1$, step $(b)$ uses $q=(2T)^{1/\log T}$, and step $(c)$ follows from the facts $q\leq e^2,\ \log T = O(\sqrt{T})$.

Hence, with probability at least $1-\frac{1}{T}$ the regret in \eqref{eq:reg_net} is bounded as
\begin{equation}
	\tilde{R}_T\leq C\left(d\sqrt{T\log T}+\frac{d^2\log^2 T}{\epsilon}\right).
\end{equation}

Next, we bound the exact regret $R_T$. Observe that the first step in our Algorithm is to use the finite $\frac{1}{T}$-net set $\mathcal{N}_{1/T}$ of actions. Thus, for any round $t\in[T]$ and any action $a\in\mathcal{A}$, there exists an action $a'\in\mathcal{N}_{1/T}$ such that $\|a-a'\|\leq\frac{1}{T}$. As a result, we get $\langle a,\theta_{*}\rangle - \langle a',\theta_{*}\rangle\leq \|a-a'\| \|\theta_{*}\|\leq \frac{1}{T}$, where $\|\theta_{*}\|\leq 1$. Hence, there is a universal constant $C$ such that we can bound the regret $R_{T}$ as
\begin{equation}
	\begin{aligned}
		R_{T} &= T\max_{a\in\mathcal{A}}\langle a,\theta_{*}\rangle -\sum_{t=1}^{T}\langle a_t,\theta_{*}\rangle\\
		&=\left[ T\max_{a\in\mathcal{A}}\langle a,\theta_{*}\rangle-T\max_{a'\in\mathcal{N}_{1/T}}\langle a',\theta_{*}\rangle \right]+ \left[ T\max_{a'\in\mathcal{N}_{1/T}}\langle a',\theta_{*}\rangle -\sum_{t=1}^{T}\langle a_t,\theta_{*}\rangle\right]\\
		&\leq T\frac{1}{T} +\tilde{R}_{T}\\
		&= 1+\tilde{R}_{T}.
	\end{aligned}
\end{equation}  
Hence, with probability at least $1-\frac{1}{T}$ the regret $R_T$ is bounded as
\begin{equation}
	R_T\leq C\left(d\sqrt{T\log T}+\frac{d^2\log^2 T}{\epsilon}\right).
\end{equation}
This concludes the proof of Theorem~\ref{thm:1}.

\begin{lemma}\label{lem:concent} Let $\hat{\theta}_i$ be the least square estimate of $\theta_{*}$ at the end of the $i$th batch of Algorithm~\ref{alg:cen}. Then, we have that
	\begin{equation}
		\Pr\left[\left|\langle a,\hat{\theta}_i-\theta_*\rangle\right|> \gamma_i\ \forall i\in [\log T]\forall a\in \mathcal{A}_i \right]\leq \frac{1}{T},
	\end{equation}
	where $\gamma_i= \sqrt{\frac{4d}{q^{i}}\log\left(4KT^2\right)}+\frac{2Bd^2+2d\log\left(4KT^2\right)}{\epsilon q^{i}}$.
\end{lemma}
\begin{proof}
	Let $\hat{\theta}_i= V_i^{-1}\sum_{a\in \mathcal{C}_i}\hat{r}_{ia}a$ be the private estimate of $\theta_{*}$ and $\overline{\theta}_i= V_i^{-1}\sum_{a\in \mathcal{C}_i}\overline{r}_{ia}a$ be the non-private estimate of $\theta_*$ as $\lbrace\overline{r}_{ia}\rbrace$ are the non-private rewards, where $V_i=\sum_{a\in \mathcal{C}_i}n_{ia}aa^\top$. From~[Chapter~$21$, Eqn~$21.1$]{}, for each $a\in\mathcal{A}_i$, we get:
	\begin{equation}~\label{eqn:first_bound}
		\Pr\left[\langle a,\overline{\theta}_i-\theta_*\rangle \geq \sqrt{2\|a\|^2_{V_i^{-1}}\log\left(\frac{1}{\beta}\right)} \right]\leq \beta,
	\end{equation} 
	where $\beta \in (0,1)$ and $\|a\|^2_{V_i^{-1}}=a^{\top} V_i^{-1}a$. Let $V_i(\pi_i)=\sum_{a\in\mathcal{C}_i}\pi_i(a)aa^{\top}$ and hence we have
	\begin{equation}
		V_i = \sum_{a\in\mathcal{C}_i} n_{ia} aa^{\top} \geq q^{i}\sum_{a\in\mathcal{C}_i}\pi_i(a) aa^{\top} = q^{i} V_i(\pi_i).
	\end{equation}
	Observe that for any symmetric random variable $x$ if $\Pr [x\geq t]\leq \beta$, then $\Pr [|x|\geq t]= \Pr [x\geq t] + \Pr [- x\geq t] \leq  2\beta$.
	Thus, from lemma~\ref{lem:core}, we have $\|a\|^2_{V_i^{-1}}= \frac{1}{q^{i}} a^{\top} V_i(\pi_i)^{-1}a\leq\frac{2d}{q^{i}}$ for each $a\in\mathcal{A}_i$. By setting $\beta=\frac{1}{4KT^{2}}$ and $\|a\|^2_{V_i^{-1}}\leq\frac{2d}{q^{i}}$ for each $a\in\mathcal{A}_i$ in~\eqref{eqn:first_bound}, we get that:
	\begin{equation}
		\Pr\left[\left|\langle a,\bar{\theta}_i-\theta_*\rangle  \right| \geq \sqrt{\frac{4d}{q^{i}}\log\left(4KT^2\right)} \right]\leq \frac{1}{2KT^2},
	\end{equation}
	for each $a\in\mathcal{A}_i$. Now, we compute the effect of the privacy in estimating $\theta_{*}$ by bounding difference $\langle a, \Bar{\theta}_i-\hat{\theta}_i\rangle$. Observe that $\hat{r}_{ia}=\Bar{r}_{ia}+z_{ia}$, where $z_{ia}\sim \mathsf{Lap}(\frac{1}{\epsilon})$, and hence, we can write $\hat{\theta}_i - \Bar{\theta}_i =V_i^{-1}\sum_{a\in\mathcal{C}_i} z_{ia}a$. Thus, for any $\alpha\in\mathcal{A}_i$, we have that:
	\begin{equation}
		\langle \alpha,\hat{\theta}_i - \Bar{\theta}_i\rangle =\sum_{a\in\mathcal{C}_i}\alpha^\top V_i^{-1}a z_{ia},
	\end{equation}  
	where $\alpha^{\top} V_i^{-1}a\leq \max_{b\in\mathcal{A}_i} \|b\|^2_{V_i^{-1}}\leq \frac{2d}{q^{i}}$ for each $a\in \mathcal{C}_i$ that holds from the fact that $V_i$ is positive semi-definite. From Lemma~\ref{lemm:laplace} presented at the end of the section, by setting $b= \epsilon$, $n=Bd$, $c=\frac{2d}{q^{i}}$, and $t=2\frac{Bd^2}{\epsilon q^{i}}+\frac{2d\log\left(4KT^2\right)}{\epsilon q^{i}}$, we get that:
	\begin{equation}
		\Pr\left[\left|\langle a,\Bar{\theta}_i-\hat{\theta}_i\rangle\right| \geq 2\frac{Bd^2}{\epsilon q^{i}}+\frac{2d\log\left(4KT^2\right)}{\epsilon q^{i}} \right]\leq \frac{1}{2KT^2},
	\end{equation}
	
	Then, by the union bound and triangle inequality we have that
	\begin{equation}
		\Pr\left[\left|\langle a,\hat{\theta}_i-\theta_*\rangle\right|> \gamma_i\ \forall i\in [\log T]\forall a\in \mathcal{A}_i \right]\leq \frac{1}{T},
	\end{equation}
	where $\gamma_i= \sqrt{\frac{4d}{q^{i}}\log\left(4KT^2\right)}+\frac{2Bd^2+2d\log\left(4KT^2\right)}{\epsilon q^{i}}$. This concludes the proof of Lemma~\ref{lem:concent}.
\end{proof}

\begin{lemma}~\label{lemm:laplace} Let $x_i=l_i z_i$ for $i\in[n]$, where $z_i\sim\mathsf{Lap}(1/b)$ and $l_i$ is constant such that $|l_i|\leq c$. Let $\Bar{x}=\sum_{i=1}^{n} x_i$. We have that
	\begin{equation}
		\begin{aligned}
			\Pr[\Bar{x}\geq t]\leq \left\{ \begin{array}{ll}
				\exp\left(-\frac{t^2b^2}{2nc^2}\right)&\text{if }\ t\leq \frac{nc}{b}\\
				\exp\left(\frac{n}{2}-\frac{b}{c}t\right)&\text{if }\ t> \frac{nc}{b}
			\end{array} 
			\right.
		\end{aligned}
	\end{equation}
\end{lemma}
\begin{proof}
	The proof follows from the concentration results of the Laplace distribution (e.g., see~{}). We have that
	\begin{equation}
		\begin{aligned}
			\Pr\left[\Bar{x}\geq t\right]&=\Pr\left[\exp\left(\lambda\Bar{x}\right)\geq e^{\lambda t}\right] \qquad \qquad & \forall\ \lambda\geq 0\\
			&\stackrel{(a)}{\leq} \frac{\mathbb{E}\left[\exp\left(\lambda\Bar{x}\right)\right]}{e^{\lambda t}}\\
			&\stackrel{(b)}{=}\frac{\prod_{i=1}^{n}\mathbb{E}\left[ e^{\lambda x_i}\right]}{e^{\lambda t}}&\\
			&\stackrel{(c)}{\leq} \frac{\prod_{i=1}^{n} e^{\lambda^{2}\frac{l_i^2}{2b^2} }}{e^{\lambda t}}\qquad\qquad &\forall\ 0\leq \lambda \leq \frac{b}{c}\\
			&\stackrel{(d)}{\leq} \frac{ e^{\lambda^{2}n\frac{c^2}{2b^2} }}{e^{\lambda t}}\qquad\qquad &\forall\ 0\leq\lambda \leq \frac{b}{c}
		\end{aligned}
	\end{equation}
	where step (a) follows from Markov's inequality and step (b) follows from the fact that $z_1,\ldots,z_n$ are independent Laplace random variables. Step (c) follows from the fact that $z_i$ is sub-exponential random variable with proxy $\frac{l_i^2}{2b^{2}}$. Step (d) follows from the fact that $|l_i|\leq c$. By choosing $\lambda =\frac{tb^2}{nc^2} $ when $t<\frac{nc}{b}$ and $\lambda = \frac{b}{c}$ when $t>\frac{nc}{b}$, we get that
	\begin{align}
		\Pr\left[\Bar{x}\geq t\right]{\leq} \left\{ \begin{array}{ll}
			\exp\left(-\frac{t^2b^2}{2nc^2}\right)&\text{if }\ t\leq \frac{nc}{b}\\
			\exp\left(\frac{n}{2}-\frac{b}{c}t\right)&\text{if }\ t> \frac{nc}{b}\\
		\end{array} 
		\right.,
	\end{align}

	This completes the proof of Lemma~\ref{lemm:laplace}. 
\end{proof}

\section{Regret and Privacy Analysis of The local DP Model (Proof of Theorem~\ref{thm:2})}~\label{sec:analysis_local}
\subsection{Privacy Analysis} The privacy proof is straightforward. For any client, since the reward is bounded by $|r|\leq 1$, the output $\hat{r}=r+\mathsf{Lap}(1/\epsilon_0)$ is $\epsilon_0$-LDP~from~\cite[Theorem~$3.6$]{Calibrating_DP06}.

\subsection{Regret Analysis} \label{subsec:local_regret}
We next prove the regret bound of Algorithm~\ref{alg:local} for stochastic linear bandits with LDP. Our proof is similar to the proofs of the central DP Algorithm presented in Section~\ref{subsec:central_regret}. 

Let $\tilde{R}_{T}$ be the regret defined in~\eqref{eq:reg_net}. Let $\mathcal{G}$ be the good event $\left\{\left|\langle a,\hat{\theta}_i-\theta_*\rangle\right|< \gamma_i\ \forall i\in [\log T]\forall a\in \mathcal{A}_i\right\}$. Lemma~\ref{lem:concent_local} shows that the event $\mathcal{G}$ holds with probability at least $1-\frac{1}{T}$. In the remaining part of the proof we condition on the event $\mathcal{G}$. When $q^i<\max\{Bd,2\log(4KT^2)\}$, the regret can be bounded by $\max\{Bd,2\log(4KT^2)\}$, and when $q^i\geq \max\{Bd,2\log(4KT^2)\}$, we bound $n_i\leq 2q^i$, and hence, $$\gamma_i\leq \sqrt{\frac{4d}{q^{i}}\log\left(4KT^2\right)}+\frac{2d}{\epsilon_0}\sqrt{\frac{\log(4KT^2)}{q^i}}\leq (1+\frac{1}{\epsilon_0})2d\sqrt{\frac{\log(4KT^2)}{q^i}}.$$ By following similar steps as in the central DP, we can show that there is universal constants $C',C$ such that the total regret in \eqref{eq:reg_net} can be bounded as
\begin{align}
	\tilde{R}_T &\leq (Bd+2\log(4KT^2))\log(T) +\sum_{i=1}^{\log T}4n_i \gamma_{i-1}\nonumber\\
	&\leq (Bd+2\log(4KT^2))\log(T) +(1+\frac{1}{\epsilon_0})2d\sum_{i=1}^{\log T} 8 q^i \sqrt{\frac{1}{q^{i-1}}\log\left(4KT^2\right)}\nonumber \\
	&\leq C'(1+\frac{1}{\epsilon_0})\left(d\sqrt{d}\log^2 (T)+ d\sqrt{d\log T}\sum_{i=1}^{\log T} q^{(i-1)/2}\right)q\nonumber \\
	&\stackrel{(a)}{\leq} C'(1+\frac{1}{\epsilon_0})q\left(d\sqrt{d}\log^2 (T)+ d\sqrt{d\log T}q^{\log T/2}\right)\nonumber \\
	&\stackrel{(b)}{\leq} C'(1+\frac{1}{\epsilon_0})q\left(d\sqrt{d}\log^2 (T)+ d\sqrt{dT\log T}\right)\nonumber \\
	&\stackrel{(c)}{\leq} C(1+\frac{1}{\epsilon_0})\left(d\sqrt{dT\log T}\right),
\end{align}
where step $(a)$ follows from the sum of a geometric series and $q>1$, step $(b)$ uses $q=(2T)^{1/\log T}$, and step $(c)$ follows from the facts $q\leq e^2,\ \log^2 T = O(\sqrt{T})$.

Hence, following similar steps as in the proof of the central DP algorithm, with probability at least $1-\frac{1}{T}$ the regret is bounded as
\begin{equation}
	R_T\leq \tilde{R}_T+1 \leq C(1+\frac{1}{\epsilon_0})\left(d\sqrt{dT\log T}\right).
\end{equation}

\begin{lemma}\label{lem:concent_local} Let $\hat{\theta}_i$ be the least square estimate of $\theta_{*}$ at the end of the $i$th batch of Algorithm~\ref{alg:local}. Then, we have that
	\begin{equation}
		\Pr\left[\left|\langle a,\hat{\theta}_i-\theta_*\rangle\right|> \gamma_i\ \forall i\in [\log T]\forall a\in \mathcal{A}_i \right]\leq \frac{1}{T},
	\end{equation}
	where $\gamma_i=\sqrt{\frac{4d}{q^{i}}\log\left(4KT^2\right)}+\frac{2d}{q^i\epsilon_0}\sqrt{n_i\log (4KT^2)}$.
\end{lemma}
\begin{proof}
	Let $\hat{\theta}_i= V_i^{-1}\sum_{a\in \mathcal{C}_i}\hat{r}_{ia}a$ be the private estimate of $\theta_{*}$ and $\overline{\theta}_i= V_i^{-1}\sum_{a\in \mathcal{C}_i}\overline{r}_{ia}a$ be the non-private estimate of $\theta_*$ as $\lbrace\overline{r}_{ia}\rbrace$ are the non-private rewards, where $V_i=\sum_{a\in \mathcal{C}_i}n_{ia}aa^\top$ and $\hat{r}_{ia}=\sum_{j=1}^{n_{ia}}\hat{r}_{ia}^{(j)}$. Similar to the central DP in Section~\ref{sec:central}, we have that
	\begin{equation}
		\Pr\left[\left|\langle a,\bar{\theta}_i-\theta_*\rangle  \right| \geq \sqrt{\frac{4d}{q^{i}}\log\left(4KT^2\right)} \right]\leq \frac{1}{2KT^2},
	\end{equation}
	for each $a\in\mathcal{A}_i$. Now, we compute the effect of the LDP in estimating $\theta_{*}$ by bounding difference $\langle a, \Bar{\theta}_i-\hat{\theta}_i\rangle$. Observe that $\hat{r}_{ia}=\sum_{j=1}^{n_{ia}}\hat{r}_{ia}^{(j)}=\Bar{r}_{ia}+z_{ia}$, where $\Bar{r}_{ia}=\sum_{j=1}^{n_{ia}}r_{ia}^{(j)}$ and $z_{ia}= \sum_{j=1}^{n_{ia}}z_{ia}^{(j)}$, where $z_{ia}^{(j)}\sim \mathsf{Lap}(\frac{1}{\epsilon_0})$. Hence, we can write $\hat{\theta}_i - \Bar{\theta}_i =V_i^{-1}\sum_{a\in\mathcal{C}_i} z_{ia}a$. Thus, for any $\alpha\in\mathcal{A}_i$, we have that:
	\begin{equation}
		\langle \alpha,\hat{\theta}_i - \Bar{\theta}_i\rangle =\sum_{a\in\mathcal{C}_i}\sum_{j=1}^{n_{ia}}\alpha^\top V_i^{-1}a z_{ia}^{(j)},
	\end{equation}  
	where $\alpha^{\top} V_i^{-1}a\leq \max_{b\in\mathcal{A}_i} \|b\|^2_{V_i^{-1}}\leq \frac{2d}{q^{i}}$ for each $a\in \mathcal{C}_i$ that holds from the fact that $V_i$ is positive semi-definite. From Lemma~\ref{lemm:laplace} presented in Section~\ref{sec:central}, by setting $b= \epsilon_0$, $n=n_i$, $c=\frac{2d}{q^{i}}$, and $t=\frac{2d}{q^{i}\epsilon_0}\sqrt{n_i\log(4KT^2)}$, we get that:
	\begin{equation}
		\Pr\left[\left|\langle a,\Bar{\theta}_i-\hat{\theta}_i\rangle\right| \geq  \frac{2d}{q^{i}\epsilon_0}\sqrt{n_i\log(4KT^2)} \right]\leq \frac{1}{2KT^2},
	\end{equation}
	Then, by the union bound and triangle inequality we have that
	\begin{equation}
		\Pr\left[\left|\langle a,\hat{\theta}_i-\theta_*\rangle\right|> \gamma_i\ \forall i\in [\log T]\forall a\in \mathcal{A}_i \right]\leq \frac{1}{T},
	\end{equation}
	where $\gamma_i= \sqrt{\frac{4d}{q^{i}}\log\left(4KT^2\right)}+\frac{2d}{q^i\epsilon_0}\sqrt{n_i\log (4KT^2)}$. This concludes the proof of Lemma~\ref{lem:concent_local}.
\end{proof}

\section{Regret and Privacy Analysis of The Shuffled Model (Proof of Theorem~\ref{thm:3})}~\label{sec:analysis_shuffled}

\subsection{Privacy Analysis}
We note that the data of each user $j$ can be represented as $\cup_{a\in \mathcal{C}_i}\{(a,r^{(j)}_a)\}$. We observe that our scheme is equivalent to performing the following steps
\begin{itemize}
	\item Each user $j\in [n_i]$ sends its data $\mathcal{D}_j=\cup_{a\in \mathcal{C}_i}\{(a,r^{(j)}_a)\}$ to the shuffler.
	\item The shuffler randomly permutes the sets $\mathcal{D}_1,...,\mathcal{D}_{n_i}$ to get $\mathcal{D}_{\pi(1)},...,\mathcal{D}_{\pi(n_i)}$.
	\item The shuffler reveals $n_i$ action reward pairs $(a_1,\hat{r}_{ia_1}),...,(a_{n_i},\hat{r}_{ia_{n_i}})$, where $(a_j,\hat{r}_{ia_j})\in \mathcal{D}_{\pi(j)}$, and $\hat{r}_{ia_j}$ is the LDP version of ${r}_{ia_j}$ ($\hat{r}_{ia_j}={r}_{ia_j}+\mathsf{Lap}(\frac{1}{\epsilon_0^{(i)}})$).
\end{itemize}
Hence, we shuffle the data, then feed it to an LDP mechanism with LDP parameter $\epsilon_0^{(i)}$ (as proved in Theorem~\ref{thm:2}). As a result, it follows from \cite{feldman2022hiding} that the output of the shuffler is $(\epsilon_i,\delta)$-DP where 
\begin{equation}
	\epsilon_i = \log\left(1+\frac{e^{\epsilon_0^{(i)}}-1}{e^{\epsilon_0^{(i)}}+1}\left(\frac{8\sqrt{e^{\epsilon_0^{(i)}}\log(4/\delta)}}{\sqrt{n_i}}+\frac{8e^{\epsilon_0^{(i)}}}{n_i}\right)\right).
\end{equation}
By the choice of $\epsilon_0^{(i)}$ as an inverse of the function $f_{n_i,\delta}$, we have that $\epsilon_i=\epsilon$ for all $i\in [\log T]$.

We observe that for any neighboring datasets $D,D'$, there is only one user data that is different between $D,D'$. That user appears in exactly one batch. It follows that Algorithm~\ref{alg:shuffle} is $(\epsilon,\delta)$-DP.

\subsection{Regret Analysis} 
We next prove the regret bound of Algorithm~\ref{alg:shuffle} for stochastic linear bandits in the shuffled model. Our proof is similar to the proofs of the LDP Algorithm presented in Section~\ref{subsec:local_regret}. 

Let $\tilde{R}_{T}$ be the regret defined in~\eqref{eq:reg_net}. Let $\mathcal{G}$ be the good event $\left\{\left|\langle a,\hat{\theta}_i-\theta_*\rangle\right|< \gamma_i\ \forall i\in [\log T]\forall a\in \mathcal{A}_i\right\}$. Lemma~\ref{lem:concent_local} shows that the event $\mathcal{G}$ holds with probability at least $1-\frac{1}{T}$. In the remaining part of the proof we condition on the event $\mathcal{G}$. When $q^i<Bd$, the regret can be bounded by $Bd$. By following similar steps as in the central DP, we can show that there is universal constants $C'$ such that the total regret in \eqref{eq:reg_net} can be bounded as
\begin{align}
	\tilde{R}_T &\leq Bd\log(T) +\sum_{i=1}^{\log T}4n_i \gamma_{i-1}\nonumber\\
	&\stackrel{(a)}{\leq} Bd\log(T)+\sum_{i=1}^{\log T} 8 q^i \sqrt{\frac{4d}{q^{i-1}}\log\left(4KT^2\right)}+C'\frac{2d}{\epsilon}\sum_{i=1}^{\log T} 8 q \sqrt{\log\left(4KT^2\right)\log(1/\delta)}\nonumber \\
	&\stackrel{}{\leq} C\left(d\sqrt{T\log T}+\frac{(d\log T)^{3/2}\sqrt{\log (1/\delta)}}{\epsilon}\right),
\end{align}
where step $(a)$ follows from the fact that from the privacy analysis, when $\epsilon_0^{(i)}\leq 1$, we get that $\epsilon = O(\epsilon_0^{(i)}\sqrt{\frac{\log(1/\delta)}{n_i}})$.

Hence, following similar steps as in the proof of the central DP algorithm, with probability at least $1-\frac{1}{T}$ the regret is bounded as
\begin{equation}
	R_T\leq \tilde{R}_T+1 \leq C\left(d\sqrt{T\log T}+\frac{(d\log T)^{3/2}\sqrt{\log (1/\delta)}}{\epsilon}\right).
\end{equation}

\section{Additional Numerical Results}~\label{app:numerics}
\begin{figure*}[t!]
	\begin{subfigure}{0.49\linewidth}
		\centerline{\includegraphics[width=0.99\linewidth, scale=0.5]{regret}}
		\caption{Central, local and shuffled models, $K=10, T=10^6$.}
		~\label{fig:regret_app}
	\end{subfigure}
	\begin{subfigure}{0.49\linewidth}
		\centerline{\includegraphics[width=0.99\linewidth, scale=0.5]{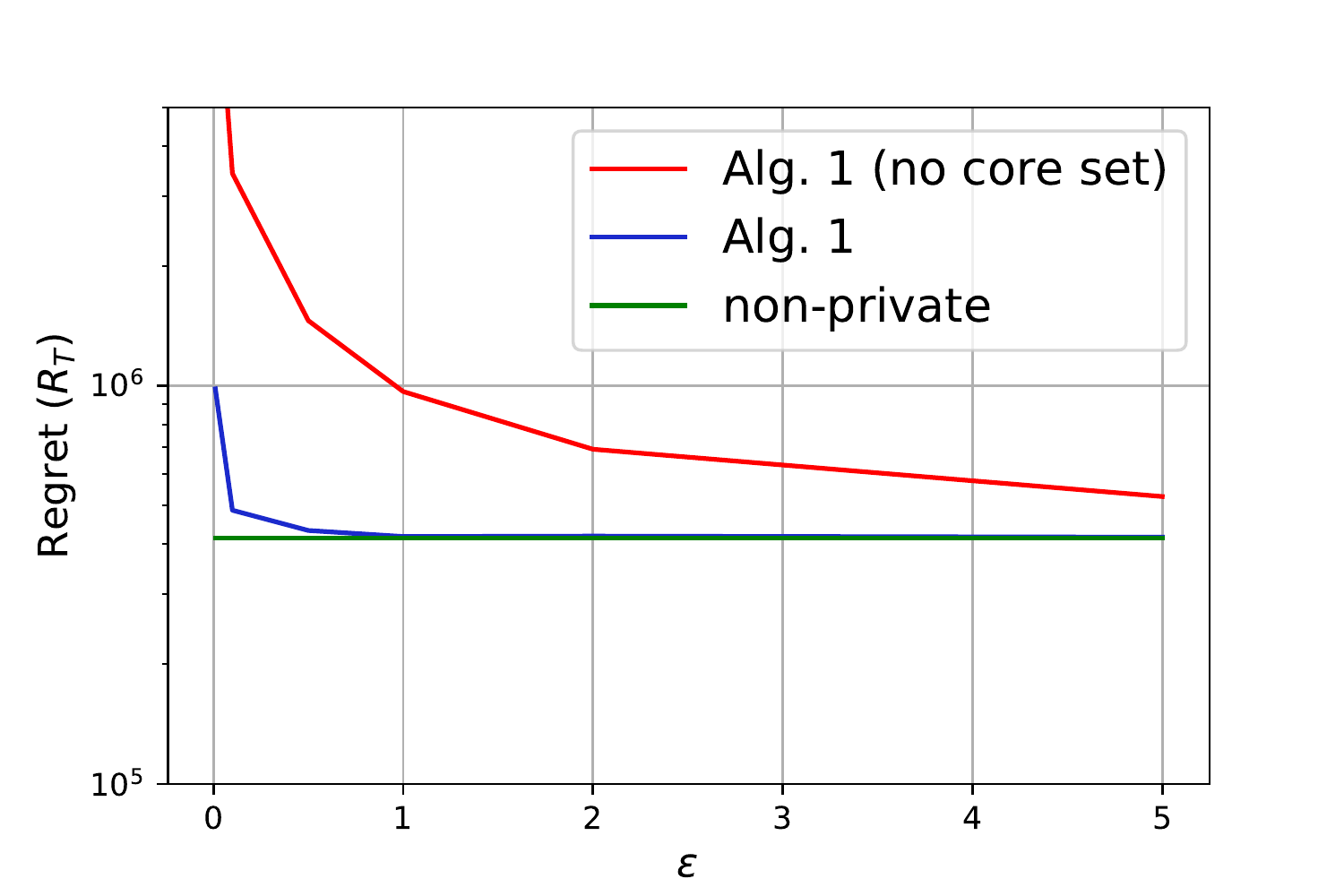}}
		\caption{Effect of core set size, $K=1000,T=10^7$.}
		~\label{fig:core_app}
	\end{subfigure}
	\caption{Regret-privacy trade-offs  for stochastic linear bandits algorithms.}\label{fig:num_app}
\end{figure*}

{\bf Data Generation.} We generate synthetic data generated as follows. The set of actions $\mathcal{A}$ contains $K$ actions, where each action $a\in\mathcal{A}$ is a $d=2$-dimensional vector. The actions $a\in\mathcal{A}$ and the optimal parameter $\theta_{*}$ are generated uniformly at random from the unit sphere $\mathcal{S}^{d-1}=\{x\in \mathbb{R}^d: ||x||_2= 1\}$. 
Figure~\ref{fig:regret_bar} plots the total regret $R_T$ over an horizon $T=10^{6}$ as a function of the privacy budget ($\epsilon$ or $\epsilon_0$ in case of LDP mechanisms). Figure~\ref{fig:num_app}  plots the total regret $R_T$ over an horizon $T$ as a function of the privacy budget ($\epsilon$ or $\epsilon_0$ in case of LDP mechanisms).

{\bf Usefulness of Core Set.} In Figure~\ref{fig:core_app}, we explore potential benefits on the performance of Algorithm~\ref{alg:cen} that use of the core set can offer. We consider   $K=1000$ and $T=10^7$, and plot the 
regret of Algorithm~\ref{alg:cen} for two cases:  (i) when we use a core set 
of size $2$-$3$ actions, similar to the dimension of our space  (labeled as Alg. 1), and (ii) when no core set is used, and instead the good set of actions of the batched algorithm is the whole action set (labeled as Alg. 1 no-core-set). We find that,  as expected from our theoretical analysis, using a core set enables to achieve performance very close to that of a non-private batched algorithm that adds no noise. In contrast, using (and adding noise to) the entire action space significantly degrades the performance.

\begin{figure}[t]
	\centering
	\includegraphics[scale=0.8]{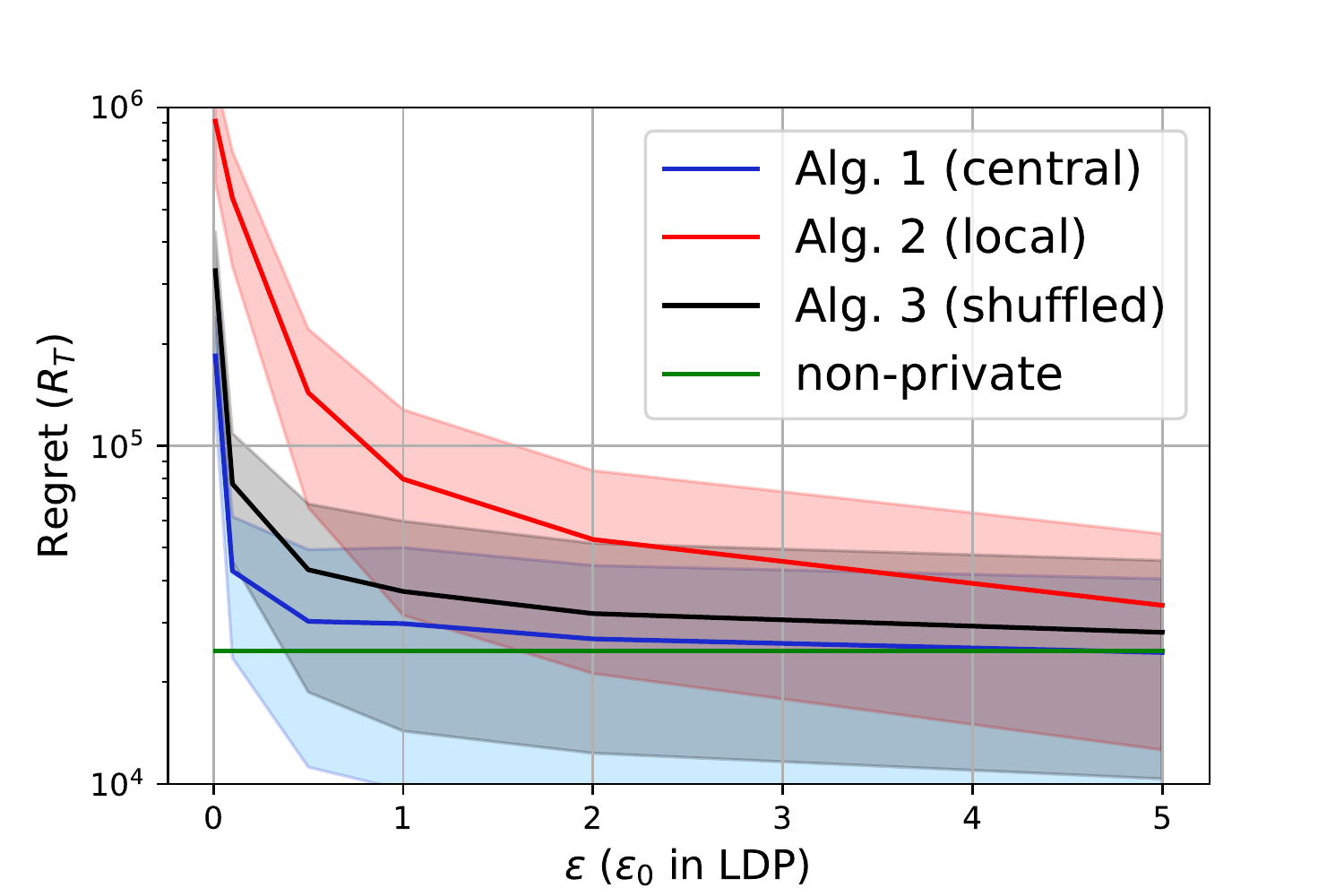}
	\captionof{figure}{Regret-privacy trade-offs for stochastic linear bandits algorithms with $T=10^6$.}
	\label{fig:regret_bar}
\end{figure}
We here present indicative results on the performance of our proposed Algorithms~\ref{alg:cen},~\ref{alg:local} and~\ref{alg:shuffle}.

{\bf Comparison of Algorithms 1, 2 and 3.} In Figure~\ref{fig:regret_app}, we compare the regret of the proposed algorithms in the central, local and shuffled models using $K=10, T=10^6$. We observe that all algorithms converge to the regret of non-private stochastic linear bandit algorithms~\cite{lattimore2020bandit} as $\epsilon\to\infty$ ($\epsilon_0\to \infty$), albeit at different rates. As predicted from the theoretical analysis, Algorithms~\ref{alg:cen} (central) and~\ref{alg:shuffle} (shuffled) offer privacy (almost) for free, closely following the non-private regret. 

}

\newpage
\bibliographystyle{IEEEtran}
\bibliography{BandPrivRefs}

\end{document}